\documentclass[sigconf]{aamas} 

\usepackage{hyperref}
\usepackage{balance} % for balancing columns on the final page
\usepackage{multicol}
\usepackage{multirow}
\usepackage{booktabs}
\usepackage{appendix}
\usepackage{comment}
\usepackage{bm}
\usepackage{dsfont}
\usepackage{algorithm}
\usepackage{algorithmic}
\usepackage{amsmath}
\usepackage{balance}

\newtheorem{theorem}{Theorem}
\newtheorem{definition}{Definition}

\usepackage[flushleft]{threeparttable}

\newcommand{\Amc}[0]{{{\mathcal{A}}}}

\newcommand{\Dmc}[0]{{{\mathcal{D}}}}

\newcommand{\Hmc}[0]{{{\mathcal{H}}}}

\newcommand{\Lmc}[0]{{{\mathcal{L}}}}
\newcommand{\Mmc}[0]{{{\mathcal{M}}}}

\newcommand{\Pmc}[0]{{{\mathcal{P}}}}

\newcommand{\Rmc}[0]{{{\mathcal{R}}}}
\newcommand{\Smc}[0]{{{\mathcal{S}}}}

\newcommand{\Xmc}[0]{{{\mathcal{X}}}}

\newcommand{\Zmc}[0]{{{\mathcal{Z}}}}

\newcommand{\piv}[0]{{\bm{\pi}}}
\newcommand{\muv}[0]{{\bm{\mu}}}

\newcommand{\thetav}[0]{{\bm{\theta}}}

\newcommand{\Ebb}{\mathbb{E}}
\newcommand{\Rbb}{\mathbb{R}}
\newcommand{\Nbb}{\mathbb{N}}

\DeclareMathOperator*{\argmax}{arg\,max}

\newcommand{\KL}{\mathrm{KL}}

\setcopyright{ifaamas}
\acmConference[AAMAS '23]{Proc.\@ of the 22nd International Conference
on Autonomous Agents and Multiagent Systems (AAMAS 2023)}{May 29 -- June 2, 2023}
{London, United Kingdom}{A.~Ricci, W.~Yeoh, N.~Agmon, B.~An (eds.)}
\copyrightyear{2023}
\acmYear{2023}
\acmDOI{}
\acmPrice{}
\acmISBN{}

\acmSubmissionID{170}

\title[Adversarial IRL for MFGs]{Adversarial Inverse Reinforcement Learning for\\ Mean Field Games}

\author{Yang Chen}
\affiliation{
\institution{Strong AI Lab, University of Auckland}
\city{Auckland}
  \country{New Zealand}}
\email{yang.chen@auckland.ac.nz}

\author{Libo Zhang}
\affiliation{
\institution{University of Auckland}
\city{Auckland}
  \country{New Zealand}}
\email{lzha797@aucklanduni.ac.nz}

\author{Jiamou Liu}
\affiliation{
\institution{University of Auckland}
\city{Auckland}
  \country{New Zealand}}
\email{jiamou.liu@auckland.ac.nz}

\author{Michael Witbrock}
\affiliation{
\institution{Strong AI Lab, University of Auckland}
\city{Auckland}
  \country{New Zealand}}
\email{m.witbrock@auckland.ac.nz}

\begin{abstract}
Goal-based agents respond to environments and adjust behaviour accordingly to reach objectives. Understanding incentives of interacting agents from observed behaviour is a core problem in multi-agent systems. Inverse reinforcement learning (IRL) solves this problem, which infers underlying reward functions by observing the behaviour of rational agents. Despite IRL being principled, it becomes intractable when the number of agents grows because of the curse of dimensionality and the explosion of agent interactions. The formalism of Mean field games (MFGs) has gained momentum as a mathematically tractable paradigm for studying large-scale multi-agent systems. By grounding IRL in MFGs, recent research attempts to push the limits of the agent number in IRL. However, the study of IRL for MFGs is far from being mature as existing methods assume strong rationality, while real-world agents often exhibit bounded rationality due to the limited cognitive or computational capacity. Towards a more general and practical IRL framework for MFGs, this paper proposes Mean-Field Adversarial IRL, a novel framework capable of tolerating bounded rationality. We build it upon the maximum entropy principle, adversarial learning, and a new equilibrium concept for MFGs. We evaluate our machinery on simulated tasks with imperfect demonstrations resulting from bounded rationality. Experimental results demonstrate the superiority of MF-AIRL over existing methods in reward recovery.
\end{abstract}


\newcommand{\BibTeX}{\rm B\kern-.05em{\sc i\kern-.025em b}\kern-.08em\TeX}

\begin{document}

%%% The following commands remove the headers in your paper. For final 
%%% papers, these will be inserted during the pagination process.

\pagestyle{fancy}
\fancyhead{}

%%% The next command prints the information defined in the preamble.

\maketitle 
%%%%%%%%%%%%%%%%%%%%%%%%%%%%%%%%%%%%%%%%%%%%%%%%%%%%%%%%%%%%%%%%%%%%%%%%
\section{Introduction}\label{sec:intro}
%=============================
% Game theory and MAIRL
%=============================

Game theory provides a general framework for predicting the strategic behaviours of interacting agents \citep{fudenberg1991game}. It concerns itself with a set of reward (or utility) functions and understands the behaviours of rational agents, i.e., the equilibrium, to be the outcome of utility-maximising strategies. An inverse problem naturally arises from this setting: how to infer the reward functions from observed behaviours? This problem is known as {\em multi-agent inverse reinforcement learning} (MAIRL) \citep{song2018multi}. More precisely, it seeks to find domain parameters for reward functions that induce observed behaviours of rational decision-makers.

%=============================
% Motivations for Large-scale MAIRL
%=============================

The recent surge in the scale of real-world multi-agent systems (MAS) \cite{garcia2002software,munz2008delay,zhou2019factorized} has raised the need for solving MAIRL in the presence of a large number of agents. In fact, we can identify at least two motivations for MAIRL in large-scale MAS. A straightforward one is as its name implies -- we wish to detect and understand the behaviour of a population of agents by modelling them using reward functions. Examples include modelling infection spread \citep{lee2021reinforcement}, discovering pricing strategies in large-scale markets \citep{subramanian2019reinforcement} and understanding the mechanism of social-norm emergence in a large population \citep{morris2019norm}. A second motivation is for the sake of designing environments for large-scale MAS so that the expected behaviour emerges if agents are rational. The behaviour of a MAS is uncertain, and unexpected behaviour is likely to arise; the growing size of the system will further exacerbate this. If one is able to pinpoint the causal relations between rewards and rational behaviours, he can manipulate the system's behaviour by tuning the reward functions.  Contrary to the examples above, applications here can be controlling and restraining infection spread, making pricing strategies in large-scale markets, and even guiding and shaping the emergence of social norms in communities, all in a data-driven manner.

%=============================
% Gaps
%=============================

However, MAIRL is notoriously intractable in the face of a large number of agents. %because of the curse of dimensionality and the explosion of agent interactions \cite{yang2018mean}. 
This is because MAIRL typically takes {\em stochastic games} \citep{levy2020stochastic} as the mathematical model, where a (Nash) equilibrium is computationally intractable when the number of agents scales to tens or hundreds \cite{daskalakis2009complexity}. 
To accommodate the need for MAIRL in large-scale MAS, we thus ask for a mathematically tractable substitute model. The recent paradigm of {\em mean field games} (MFGs) \citep{lasry2007mean,huang2006large} achieves tractability by borrowing the idea of mean field approximation from statistical physics to simplify agent interactions. It takes the limit as the number of agents approaches infinity and reduces the whole-system interactions to those between a single individual agent and the {\em mean field}, a virtual agent that represents aggregated behaviour of the population at large. This dual-view interplay gives rise to {\em mean field Nash equilibrium} (MFNE), which stipulates bidirectional constraints between the two sides: every agent's policy maximises its rewards given the mean field and, in turn, the mean field is uniquely resulted by all agents' policies. Importantly, MFNE is shown to be an approximate Nash equilibrium in the corresponding finite-agent stochastic games \cite{saldi2018markov}. 
To break through the limitation of agent number in MAIRL, it is thereby promising to transfer the concept of IRL to MFGs.

Unfortunately, IRL remains largely unexplored in MFGs, albeit there are two recent attempts. \citet{yang2018learning} first proposed a {\em centralised} IRL method for MFGs by showing that an MFG can be reduced to a {\em Markov decision process} (MDP) that describes the population's collective behaviour and average rewards; they thus applied single-agent IRL methods on top of this MDP. Subsequently, \citet{chen2022individual} revealed that this reduction holds only if the MFNE is {\em socially optimal}, i.e., it maximises the population's average rewards; they thus framed the problem at the {\em decentralised} setting, i.e., inferring the reward function for an individual agent from the observed individual behaviour rather than the population behaviour. They put forward {\em Mean Field IRL} (MFIRL), %based on {\em margin optimisation} \citep{ratliff2006maximum}, 
a more general approach effective for both socially optimal and ordinary MFNE. 

However, both methods above are still limited in terms of practical use. %which arise from the solution concept of MFNE. 
First, due to limited cognitive or computational capability, an agent often has {\em bounded rationality} in real life, i.e., choosing satisfactory rather than optimal actions \citep{harstad2013bounded}.  Consequently, the resulting behaviour possesses uncertainties in general. For example, a customer in a restaurant orders an acceptable dish which is not necessarily his favourite as he is rush in time. These two methods cannot reason about such uncertainties as agents in MFNE are assumed never to take suboptimal behaviour. This makes both methods unsuitable for situations where the agents are bounded rational. %lacks of the ability to tolerate such uncertainties, its use in real-world applications is limited. 
Then, since an MFNE is not unique in general \citep{cui2021approximately}, the observed behaviour may be insufficient for us to learn a reward function that determines a unique policy. In this sense, from an application standpoint, the existing methods are also not useful for environment design for large-scale MAS.

%=============================
% Challenges 
%=============================
Towards a more general and practical IRL method for MFGs, we invoke the idea of Maximum Entropy IRL (MaxEnt IRL) \citep{ziebart2008maximum,ziebart2010modeling}, which provides a general probabilistic framework to tackle bounded-rational behaviour. It is state-action trajectory-centric and assumes the observed trajectories follow a distribution (in terms of rewards) with the maximum entropy. It thus allows us to find a reward function that rationalises observed behaviour with the {\em least commitment}. Moreover, since a policy leading to the maximum entropy trajectory distribution is unique given a reward function, MaxEnt IRL is more useful for environment design.
However, extending MaxEnt IRL to MFGs is challenging. First, since MFNE assumes agents never take suboptimal behaviour, it is incompatible with MaxEnt IRL in the sense that it cannot provide a trajectory distribution that can be used for the probabilistic reward inference. Second, since in MFGs, the individual and population dynamics are intertwined (the policy and mean field are interdependent), the trajectory distribution is intractable to express in terms of rewards, which would prevent us from performing probabilistic inference for reward functions. %Consequently, it would prevent us from performing probabilistic inference for reward functions.

%=============================
% Summarise Contributions 
%=============================

The primary contribution of this paper lies in the proposal of a new probabilistic IRL framework, {\em Mean-Field Adversarial IRL} (MF-AIRL), for MFGs. MF-AIRL integrates ideas from decentralised IRL for MFGs, MaxEnt IRL, and adversarial learning \citep{goodfellow2014generative} into a unified probabilistic model for reward inferences in large-scale MAS. 
We summarise specific contributions as follows: 
{\bf (1)} We build MF-AIRL upon a new equilibrium concept termed {\em entropy-regularised MFNE} (ERMFNE) %which augments the rewards with the {\em causal entropy} \citep{ziebart2010modeling} 
(see Sec.~\ref{sec:ERMFNE}). We show that ERMFNE can characterise an individual's trajectory distribution induced by a reward function in a principled way (see Theorem~\ref{thm:trajectory_ERMFNE}). 
{\bf (2)} Taking ERMFNE as the solution concept, we extend MaxEnt IRL to MFGs (see Sec.~\ref{sec:extension}). We decouple individual and population dynamics by deriving the empirical value of the mean field from the observed behaviour (see Theorem~\ref{thm:consistent}). %This extension is not trivial because, in MFGs, individual and population dynamics %(policies and mean fields) are intertwined, hindering the probabilistic inference for the reward function. We address this challenge by devising a method to decouple individual approximately and population dynamics %determining the mean field as an empirical value estimated from demonstrations 
{\bf (3)} By using adversarial learning to solve MaxEnt IRL in MFGs efficiently, we develop the practical MF-AIRL framework (see Sec.~\ref{sec:MF-AIRL}).
% {\bf (4)} We relate our MF-AIRL to existing reinforcement learning approaches for MFGs and IRL methods for Markov games and discuss in technical detail its advantages against the aforementioned existing IRL methods for MFGs (see Sec.~\ref{sec:related}). %, and further explore its generalising abilities (see Sec.~\ref{sec:related}). 
{\bf (4)} We evaluate MF-AIRL on tasks that simulate scenarios of marketing strategy making, virus propagation modelling and social norm emergence, all on a large scale %with cooperative and competitive environments %and simulated battle games \citep{zheng2018magent}, both 
 (see Sec.~\ref{sec:experiments}). Results demonstrate the outperformance of MF-AIRL over existing methods in reward recovery.

\section{Preliminaries}\label{sec:pre}
This section introduces {\em mean field games} (MFGs) and the {\em maximum entropy inverse reinforcement learning} (MaxEnt IRL) framework. The marriage of the two gives rise to our proposed multi-agent IRL approach dedicated to large-scale multi-agent systems.

\subsection{Mean Field Games}
Following the conventional MFG model in the learning setting, we focus on MFGs with finite state and action spaces and, more generally, a finite time horizon \citep{elie2020convergence}. %We will later argue that our method requires no extra effort to be applied to more general cases. 
Consider an $N$-player symmetric game, i.e., all agents share the same {\em local state} space $\Smc$, {\em action} space $\Amc$, and a reward function that is invariant under the permutation of identities of agents without changing their states and actions.
Let $(s^1,\ldots,s^N)\in \Smc^N$ denote a {\em joint state}, where $s^i\in \Smc$ is the state of the $i$th agent. As $N$ grows large, the game becomes intractable to analyse due to the curse of dimensionality. MFGs achieve tractability by considering the asymptotic limit when $N$ approaches infinity. Formally, take the limit as $N \to \infty$. Due to the homogeneity of agents, MFGs use an empirical distribution $\mu \in \Pmc(\Smc)$, called a {\em mean field}, to represent the statistical information of the joint state: %$\mu(s) \triangleq  \frac{1}{N} \sum_{i=1}^N \mathds{1}_{\{s^i = s\}}$,
\begin{equation*}
	\mu(s) \triangleq \lim_{N \to \infty} \frac{1}{N} \sum_{i=1}^N \mathds{1}_{\{s^i = s\}}.
\end{equation*}
%rather than modelling each agent individually. 
Here, $\Pmc(\Smc)$ represents the set all probability measures over $\Smc$ and $\mathds{1}$ denotes the indicator function, i.e., $\mathds{1}_{x} = 1$ if $x$ is true and $0$ otherwise. The {\em transition function} $P\colon \Smc \times \Amc \times \Pmc(\Smc) \times \Pmc(\Smc) \to [0,1] $ specifies how states evolve, i.e., an agent transits to the next state $s_{t+1}$ with the probability $P(s_{t+1} \vert s_t, a_t, \mu_t)$. %conditioned on its current state, action, and mean field. %This function also induces a transition of the mean field at the population level, which maps a current mean field to the next mean field based on all agents' current states and actions. 
Let $T\in \Nbb^+$ denote a finite time horizon. A {\em mean field flow} (MF flow for short) thus consists of a sequence of $T + 1$ mean fields $\muv \triangleq \{\mu_t\}_{t=0}^T$. The initial mean field $\mu_0$ is given. %and each $\mu_t$  ($0<t< T$) is the empirical distribution obtained by applying the transition above to $\mu_{t-1}$. 
%Let $s_0,a_0\ldots,s_T,a_T$ denotes a {\em trajectory} of states and actions of a fixed agent. 
The {\em running reward} at each step is determined by a bounded {\em reward function} $r\colon \Smc \times \Amc \times \Pmc(\Smc) \rightarrow \Rbb$. Let $\tau \triangleq \{ (s_t, a_t) \}_{t=0}^T$
denote a state-action {\em trajectory} of an agent. We write an agent's long-term reward under a given MF flow $\muv$ as  %the discounted sum of per-step rewards: 
$\Rmc(\tau) \triangleq \sum_{t=0}^{T-1} r(s_t,a_t,\mu_t)$.\footnote{Following the convention \citep{yang2018learning,elie2020convergence,chen2022individual}, we set the reward at the last step ($t=T$) as $0$.} %where $\gamma \in (0,1]$ is the {\em discounted factor}. 
In summary, an MFG is defined as a tuple $(\Smc, \Amc, P, \mu_0, r)$. % \cite{elie2020convergence}. 

A time-varying stochastic {\em policy} $\piv \triangleq \{\pi_t\}_{t=0}^T$ %\footnote{Some works consider the stationary policy in MFG which results in a stationary MFNE \cite{subramanian2019reinforcement,guo2019learning}. For generality, we adopt the time-varying policy in this paper.} 
is adopted to characterise a strategic agent, where each $\pi_t\colon \Smc\rightarrow \Pmc(\Amc)$ is the {\em per-step policy}, i.e., an agent chooses actions following  $a_t\sim \pi_t(\cdot\vert s)$.   
Given an MF flow $\muv$ and a policy $\piv$, an agent's {\em expected return} (cumulative rewards) is written as 
\begin{equation}
	J(\muv, \piv) \triangleq \Ebb_{\tau\sim(\muv, \piv)} \left[ \Rmc(\tau) \right],
\end{equation}
where $s_0 \sim \mu_0, a_t \sim \pi_t(\cdot \vert s_t),s_{t+1} \sim P(\cdot \vert s_t, a_t, \mu_t).$

\subsection{Mean Field Nash Equilibrium}\label{sec:MFNE}

%An agent seeks an optimal %control in the form of a 
%policy to maximise the expected return. %If a MF flow $\muv$ is fixed, we will derive an induced MDP with a non-stationary transition function. %and the policy adopted by a single agent has negligible impact on the population.
Fixing an MF flow $\muv$, a policy $\piv$ is called a {\em best response} to $\muv$ if it maximises $J(\muv, \piv)$.
We denote the set of all best-response policies to $\muv$ by $\Psi (\muv) \triangleq \argmax_{\piv} J(\muv, \piv)$. However, since all agents optimise their policies simultaneously, the MF flow would shift. Thus, the solution must consider how a policy at the individual level affects the MF flow at the population level. 
Due to the homogeneity of agents, everyone would follow the same policy at optimality. %Under this assumption, %the dynamics of MFG is governed by a set of two equations: the {\em Bellman equation} models the {\em backward dynamics} of value functions at agent level:
%\begin{equation}\label{eq:Bellman}
%	V_t^{\muv, \piv}(s)  =  \mathbb{E}_{\muv, \piv} \left[ r(s,a, \mu_t) + \gamma  V_{t+1}^{\muv, \piv}(s')\right];
%\end{equation}
%the {\em McKean-Vlasov} (MKV) equation \cite{carmona2013control} models {\em forward dynamics} of MF flow at population level:
The dynamics of the MF flow can thus be governed by the (discrete-time) {\em McKean-Vlasov} (MKV) equation \citep{carmona2013control}:
\begin{equation}\label{eq:MKV}
	\mu_{t+1}(s') = \sum_{s \in \Smc} \mu_t(s) \sum_{a \in \Amc} \pi_t(a \vert s)\; P(s' \vert s, a, \mu_t).
\end{equation}
Given a policy $\piv$, denote $\muv = \Phi(\piv)$ as the MF flow that fulfils MKV equation.  We say $\muv$ is {\em consistent} with $\piv$ if $\muv = \Phi(\piv)$. This consistency guarantees that a single agent's state marginal distribution flow matches the MF flow at the population level. %\citep{guo2019learning}. 
The conventional solution concept for MFGs is the {\em mean field Nash equilibrium}. %as is defined below. %, where all agents adopt the same policy that best responds to the MF flow, which is consistent with the policy.

\begin{definition}\label{def:MFE}
	A pair $(\muv^{\star}, \piv^{\star})$ is called a {\em mean field Nash equilibrium} (MFNE) if it satisfies:
	\begin{enumerate}
		\item {\bf\em Agent rationality:} $\piv^\star \in \Psi(\muv^\star)$;
		\item {\bf\em Population consistency:} $\muv^{\star} = \Phi(\piv^{\star})$.
	\end{enumerate} 
\end{definition}
Computing an MFNE typically requires a fixed-point iteration procedure for the MF flow \citep{guo2019learning}. Formally, through defining any mapping $\hat{\Psi}: \muv \mapsto \piv$ that identifies a best response in $\Psi(\muv)$, we obtain a fixed point iteration for $\muv$ by alternating between $\piv = \hat{\Psi}(\muv)$ and $\muv = \Phi(\piv)$. The assumption for the uniqueness of MFNE is that the fixed-point iteration will converge to a unique $\muv$ \citep{guo2019learning}. 
However, the fixed-point iteration is not guaranteed to converge to a unique $\muv$, and multiple MFNE can coexist \citep{cui2021approximately}. 

\subsection{Maximum Entropy IRL}\label{sec:MaxEnt}

We next give an overview of MaxEnt IRL \citep{ziebart2008maximum,ziebart2010modeling} in the context of a Markov decision process (MDP) defined by a tuple $(\Smc, \Amc, P, \rho, r)$, where %$\S,\A,\gamma$ are state space, action space and discounted factor respectively, and 
$r(s,a)$ is the reward function, and the environment dynamics is determined by the transition function $P(s' \vert s,a)$ and initial state distribution $\rho(s)$. 
In (forward) reinforcement learning (RL), an optimal policy may not exist uniquely. MaxEnt RL solves this ambiguity by augmenting the expected return with a {\em causal entropy} \footnote{Throughout the rest of the paper, we refer to the term entropy as the causal entropy.} \citep{ziebart2010modeling} regularisation term $\Hmc(\pi) \triangleq \mathbb{E}_{\pi} [ - \log \pi(a \vert s) ]$, i.e., the objective is to find a (stationary) policy $\pi^\star$ such that 
\begin{equation*}
\pi^\star = \argmax_{\pi} \mathbb{E}_{\tau \sim \pi} \left[ \sum_{t=0}^{T-1}  r(s_t,a_t) + \beta \Hmc(\pi(\cdot \vert s_t)) \right],	
\end{equation*}
where $\tau$ is a state-action trajectory sampled via $s_0 \sim \rho_0$, $a_t \sim \pi(\cdot \vert s_t)$, $s_{t+1} \sim P(\cdot \vert s_t, a_t)$ and $\beta > 0$ controls the relative importance of reward and entropy. % Without loss of generality, it is often assumed that $\beta = 1$ \citep{yu2019meta}. % With suitably large $\beta$, the policy with both the highest expected return and the highest causal entropy is unique  %since the policy entropy is strictly concave concerning policies and the set of optimal policies is convex. %Shown in \cite{ziebart2010modeling,haarnoja2017reinforcement}, the optimal policy takes the form of $\pi^*(a_t \vert s_t) \propto \exp(\frac{1}{\beta} Q^{\pi^*}_{\soft}(s_t,a_t))$, where $Q^\pi_{\soft}(s_t, a_t) \triangleq r(s_t, a_t) + \mathbb{E}_{\pi} [ \sum_{\ell = t+1}^{T-1} \gamma^{\ell - t} ( r(s_\ell, a_\ell) +  \H(\pi(\cdot \vert s_\ell)) )]$ is the {\em soft $Q$-function}. It can be seen that $\beta$ essentially embodies the temperature in an energy model. %Intuitiely, $\pi^\ast$ tends to maximise expected returns with the least bias on environments. 

Suppose we have no access to the reward function but have a set of observed trajectories sampled from an {\em unknown} expert policy $\pi^E$ obtained by the above MaxEnt RL procedure. %Also assume $\D$ contains all knowledge that an expert can provide, i.e., we cannot ask for additional information from the expert.  
MaxEnt IRL aims to infer a reward function that rationalises the observed behaviour, which reduces to the following maximum likelihood estimation (MLE) problem (assume $\beta = 1$ \citep{yu2019meta}):
\begin{equation}\label{eq:trajectoryMaxEntIRL}
	p_\omega(\tau) \propto  \rho(s_0) \cdot \prod_{t=0}^{T-1} P(s_{t+1} \vert s_t, a_t)  \cdot e^{R_\omega(\tau)},
\end{equation}
\begin{equation}\label{eq:MaxEntIRL}
	\max_{\omega} \mathbb{E}_{\tau \sim \pi^E} \left[\log p_\omega(\tau)\right] = \mathbb{E}_{\tau \sim \pi^E} \left[ R(\tau) \right] - \log Z_\omega.
\end{equation}
Here, $R_\omega(\tau) \triangleq \sum_{t=0}^{T-1} r_\omega(s_t, a_t)$ where $r_\omega$ is an $\omega$-parameterised reward function, and $Z_\omega \triangleq \sum_{\tau \sim \pi^E} e^{R_\omega(\tau)}$ is the {\em partition function} of the distribution defined in Eq.~\eqref{eq:trajectoryMaxEntIRL}, i.e., a summation over all feasible trajectories. %The initial distribution $\rho$ and transition function $P$ are omitted in $Z_\omega$ because they do not depend on $\omega$. 
Exactly computing $Z_\omega$ is intractable if the state-action space is large. % or continuous. %It is common to estimate $Z_\omega$ by {\em importance sampling} using a sampler $\pi_\theta$ (a policy):%Many methods for estimating $Z_\omega$ are proposed and most are based on {\em importance sampling} \cite{boularias2011relative,finn2016guided,fu2018learning}.
%\begin{equation*}%\label{eq:importance_sampling}
%	Z_\omega \approx \Ebb_{\tau \sim \pi_\theta} \left[ \exp \left( \sum_{t=0}^{T-1} \gamma^t r_\omega (s_t, a_t) \right) \Bigg/ \prod_{t=0}^{T-1} \pi_\theta (a_t \vert s_t) \right].
%\end{equation*}

{\em Adversarial IRL} (AIRL) was proposed by \citep{fu2018learning} as an efficient sampling-based approximation to MaxEnt IRL, which reframes Eq.~\eqref{eq:MaxEntIRL} as 
optimising a  {\em generative adversarial network} \citep{goodfellow2014generative}. It uses a discriminator $D_\omega$ (a binary classifier) and a {\em adaptive sampler} $\pi_\theta$ (a policy). Particularly, the discriminator takes the following form: 
%\begin{equation*}\label{eq:AIRL}
	$$D_\omega(s,a) = \frac{e^{ f_\omega(s,a)}}{e^{ f_\omega(s,a) + \pi_\theta(a \vert s)}},$$
%\end{equation*}  
%where the adaptive sampler $\pi_\theta(a \vert s)$ is pre-computed as input to the discriminator. 
where $f_\omega$ serves as the parameterised reward function. The update of $D_\omega$ is interleaved with the update of $\pi_\theta$: $D_\omega$ is trained to update the reward function by distinguishing between the trajectories sampled from the expert and the adaptive sampler; while $\pi_\theta$ is trained to maximise $$\Ebb_{\tau\sim\pi_\theta} \left[ \sum_{t=0}^{T-1} \log D_\omega(s_t, a_t) - \log (1 - D_\omega (s_t, a_t)) \right].$$
IRL faces the ambiguity of {\em reward shaping} \citep{ng1999policy}, i.e., multiple reward functions can induce the same optimal policy. To mitigate this ambiguity, \citet{fu2018learning} further restrict the parameterised reward in $D_\omega$ to a specific structure by supplying a state-only {\em potential-based reward shaping} function $h_\phi: \Smc \to \Rbb$:
\begin{equation*}\label{eq:shaping}
	f_{\omega,\phi}(s_t, a_t, s_{t+1}) = r_\omega(s_t, a_t) + h_\phi(s_{t+1}) - h_\phi(s_t).
\end{equation*}
Shown in \cite{fu2018learning}, under certain conditions, $r_\omega(s,a) + h_\phi(s)$ will recover the ground-truth reward function up to a constant.
%It is shown that PBRS is the sufficient and necessary condition to ensure policy (equilibrium) invariance for both MDP \cite{ng1999policy} and stochastic games \cite{devlin2011theoretical}.

\subsection{IRL for MFGs}

We adopt the general decentralised IRL setting for MFGs as in \citep{chen2022individual}, which aims to infer the individual reward function from observed individual behaviour. More formally, let $(\Smc, \Amc, P, \mu_0, r)$ be an MFG. Suppose we do not know $r(s,a,\mu)$ but have a set of $M$ observed expert behaviour $\Dmc_E =  \{ \tau_j \}_{j = 1}^M$ sampled from an {\em unknown} equilibrium $(\muv^E, \piv^E)$, where each $\tau = \{(s_t, a_t)\}_{t=0}^{T}$ is an individual agent's state-action trajectory sampled via $s_0 \sim \mu_0$, $a_t \sim \pi_t(\cdot \vert s_t)$, $s_{t+1} \sim P(\cdot \vert s_t, a_t, \mu_t)$. %Following the convention in IRL \citep{ho2016generative,song2018multi,chen2022individual}, we assume that $\Dmc_E$ provides the entire supervision signals, i.e., we cannot further communicate with experts for additional information. 
IRL for MFG asks for a reward function $r(s,a,\mu)$ under which $(\muv^E, \piv^E)$ constitutes an equilibrium.

\section{Entropy-Regularised MFNE}\label{sec:ERMFNE}

This section introduces and justifies a new solution concept for MFGs, which allows us to characterise the bounded rationality of agents and thereby extend MaxEnt IRL to MFGs. %based on which we extend MaxEnt IRL to MFGs. Based on this extension, we finally put forward MF-AIRL, a practical IRL framework for MFGs.

\subsection{A New Equilibrium Concept for MFGs}%\label{sec:ERMFNE}

To extend MaxEnt IRL to MFGs, we need to characterise the trajectory induced by a reward function with a particular distribution as analogous to Eq.~\eqref{eq:trajectoryMaxEntIRL}. %, which can be used to maximise the likelihood of demonstrated trajectories. 
However, MFNE cannot {\em explicitly} define a tractable trajectory distribution 
%due to the aforementioned policy ambiguity: the contractivity of the MFNE operator fails if multiple MFNE exist, and even if the contractivity holds, it remains ambiguous which to identify if there exist multiple best-response policies to a MF flow.
as it requires agents never to take suboptimal actions, whereas, in MaxEnt IRL, an agent can take sub-optimal actions with certain low probabilities. %as a result of the entropy regularisation. 
%This difference can be explained from the viewpoint of rationality: because of limited cognitive or computational capability, an agent often has {\em bounded rationality}, i.e., choosing satisfactory  rather than optimal actions \citep{harstad2013bounded}. Consequently, the resulting demonstrations possess uncertainties in general. %MaxEnt IRL can characterise such uncertainties using an energy-based formulation in terms of rewards, but IRL methods for MFGs based on MFNE would not be able to. 
%In this sense, the lack of the ability to explicitly provide a trajectory distribution can be interpreted as powerlessness to reason about uncertainties in expert demonstrations. %, in the form of noises or suboptimal behaviours caused by the bounded rationality of experts. %, in order to capture nearly rational (or computationally bounded) experts \cite{yu2019multi}.
To bridge the gap between MFGs and MaxEnt IRL, we need a ``soft'' equilibrium concept that can characterise uncertainties in observed behaviour as against MFNE being a ``hard'' equilibrium. To this end, a natural way in game theory is to incorporate policy entropy into rewards \citep{ortega2011information,gabaix2014sparsity},     which enables {\em bounded rationality}, i.e., agents can take satisfactory rather than optimal actions. % in the same vine of Eq.~\eqref{eq:maxent_policy}. 
This inspires a new solution concept -- {\em entropy-regularised MFNE} (ERMFNE) -- where an agent aims to maximise the entropy-regularised rewards:
\begin{equation*}
    \tilde{J}(\muv, \piv) \triangleq \mathbb{E}_{\tau\sim(\muv, \piv)} \left[ \sum\limits_{t = 0}^{T-1}  r(s_t, a_t, \mu_t) + \beta \Hmc(\pi_t(\cdot \vert s_t)  \right].
\end{equation*}
%\begin{center}
%\resizebox{.6\textwidth}{!}{
%$\tilde{J}(\muv, \piv) \triangleq \mathbb{E}_{\muv, \piv} \left[ \sum\limits_{t = 0}^{T-1} \gamma^{t} \big(  r(s_t, a_t, \mu_t) + \beta \H(\pi_t(\cdot \vert s_t) \big) \right].$
%}
%\end{center}
%By using $\tilde{\Psi}(\muv) \triangleq \max_{\piv} \tilde{J} (\muv, \piv)$ to denote the set of all best-response policies that maximises $\tilde{J}(\muv, \piv)$ given an MF flow $\muv$,\footnote{We slightly abuse the term ``best response'' to denote the policy that maximises the entropy-regularised rewards.} we formally define ERMFNE %Particularly, we will shortly show that trajectories induced by an ERMFNE can be characterised by an energy-based model and can thus be used for the probabilistic inference of the underlying reward function.} as follows. 

\begin{definition}
	A pair of MF flow and policy $(\tilde{\muv}^\star, \tilde{\piv}^\star)$ is called an {\em entropy-regularised MFNE (ERMFNE)} if it satisfies: 
	\begin{enumerate}
		\item {\bf\em Agent bounded rationality:} $\tilde{J} (\tilde{\muv}^{\star}, \tilde{\piv}^{\star}) = \max_{\piv} \tilde{J} (\tilde{\muv}^{\star}, \piv)$;
		\item {\bf\em Population consistency:} $\tilde{\muv}^{\star} = \Phi(\tilde{\piv}^{\star})$.
	\end{enumerate}  
\end{definition}

\begin{comment}
\begin{remark}
	{ Despite the entropy-regularised MFGs have been studied  \citep{cui2021approximately,anahtarci2020q,guo2020entropy}, existing works are motivated from a computational perspective: entropy regularisation can improve the stability and the convergence of algorithms for computing an equilibrium because it can bring a good property of uniqueness. %, as the contractivity of the fixed point iteration for the MF flow (recall from Sec.~\ref{sec:MFNE}) is more likely to hold with the entropy regularisation (elaborated in Sec.~\ref{sec:existence-uniqueness}). 
	While, we take a step further to show that the entropy regularisation endows the resulting equilibrium concept with the ability to characterise uncertainties in agent behaviours. Particularly, we will shortly show that trajectories induced by an ERMFNE can be characterised by an energy-based model and can thus be used for the probabilistic inference of the underlying reward function.
	} 
	%This distinct motivation sets our work apart from these existing works.
\end{remark}
\end{comment}
%Analogous to MFNE, ERMFNE exists for any temperature $\beta > 0$ if the reward function and transition function are continuous and bounded Shown in \cite[Proposition~3]{cui2021approximately}. Using $\tilde{\Psi}(\muv) = \tilde{\piv}$ to denote the policy specified by Eq.~\eqref{eq:softQMFG}, we obtain the {\em ERMFNE operator} $\tilde{\Gamma} = \Phi \circ \tilde{\Psi}$. 

Despite the entropy-regularised MFGs have been studied in the literature \cite{anahtarci2020q,guo2020entropy,cui2021approximately}, existing work is motivated from a computational perspective, that is, entropy regularisation can improve the stability and the convergence of algorithms for computing an equilibrium. Particularly, \citet{cui2021approximately} independently proposed a similar solution concept and showed that entropy regularisation relaxes the condition for uniqueness as opposed to the unregularised case: {\bf (1)} With entropy regularisation, the best-response policy $\tilde{\piv}$ to an MF flow $\muv$ exists {\em uniquely}; \footnote{Hereafter, we will slightly abuse the term ``best response'' to denote the policy that maximises the entropy-regularised rewards.} {\bf (2)}  ERMFNE exists for any $\beta > 0$ if $r(s,a,\mu)$ and $P(s' \vert s,a,\mu)$ are continuous. Using $\tilde{\piv} = \tilde{\Psi}(\muv)$ to denote the unique best-response policy to $\muv$, we obtain the the fixed point iteration for ERMFNE by alternating between $\piv = \tilde{\Psi}(\muv)$ and $\muv = \Phi(\piv)$. The fixed point iteration converges to a unique MF flow if $\beta$ is large (according to the reward function scale), thereby implying a {\em unique} ERMFNE. 

Note that in ERMFNE, we recover optimality (MFNE) if $\beta=0$.
Although the optimality and uniqueness are approached respectively at two extremes of $\beta$, in this paper, we prioritise the property of uniqueness to ensure the well-definedness of IRL for MFGs, i.e., we assume trajectories are observed from a unique ERMFNE (with a suitable $\beta$) so that the observed behaviour can be interpreted with a unique equilibrium. Since we can adjust the relative importance between rewards and entropy by scaling reward functions, following the convention in MaxEnt IRL \citep{fu2018learning,yu2019multi,yu2019meta} and without loss of generality, we assume $\beta = 1$ in the remainder of the analysis.

\subsection{Trajectory Distributions under ERMFNE}\label{sec:traj_dist} 
 %Also note that we recover optimality as $\beta \to 0$. %Though we cannot simultaneously have $\beta \to \infty$ and $\beta \to 0$, empirically we can often find low $\beta$ that achieves convergence \cite{cui2021approximately}. 
%Since we can always adjust the relative importance of reward and entropy by scaling reward functions, without loss of generality, we assume $\beta = 1$ in the remainder of analysis.
%The parameter $\beta$ controls the degree of rationality of agents. Note that we recover MFNE as $\beta \to 0$. Shown by \citep{cui2021approximately}, the fixed point iteration $\piv = \tilde{\Psi}(\muv)$, $\muv = \Phi(\piv)$ converges to a unique MF flow if $\beta$ is sufficiently large and $r(s,a,\mu)$ and $p(s' \vert s,a,\mu)$ are continuous, which implies the existence and uniqueness of ERMFNE. Although the optimality and uniqueness are approached respectively at two extremes of $\beta$, in this paper, we prioritise the property of uniqueness to ensure the well-definedness of IRL for MFG, i.e., the expert behaviours can be interpreted with a unique equilibrium. Since we can always adjust the relative importance between rewards and entropy by scaling reward functions, without loss of generality, we assume $\beta = 1$ in the remainder of the analysis.
%While, in this paper, we take a step further to show that the entropy regularisation allows us to use the resulting equilibrium concept to characterise uncertainties in agent behaviours; it can thus be used for the probabilistic inference of the underlying reward function.

Besides uniqueness, in this paper, we take a step further to show that the entropy regularisation endows ERMFNE with the capability of reasoning about uncertainties in observed behaviour in a principled way. Specifically, we show that trajectory distributions induced by ERMFNE can be characterised by an {\em energy-based model}, i.e., trajectories with high expected cumulative rewards are generated with exponentially high probabilities, as is illustrated in Fig.~\ref{fig:trajectory}. It can thus be used for the probabilistic reward inference. %of the underlying reward function. %Formally, we can characterise trajectory distributions induced by ERMFNE with the following theorem.

\begin{theorem}\label{thm:trajectory_ERMFNE}
Let $(\tilde{\muv}^\star, \tilde{\piv}^\star)$ be the ERMFNE for an MFG $(\Smc, \Amc, P,$ $\mu_0, r)$, and $D_{\mathrm{KL}}$ denote the Kullback-Leibler (KL) divergence. Then, $(\tilde{\muv}^\star, \tilde{\piv}^\star)$ is the optimal solution to the following constrained optimisation problem:
	\begin{equation*}
		\min_{\muv, \piv}D_{\mathrm{KL}}\left(p_1(\tau) \parallel p_2(\tau)\right)  \text{ s.t. } \muv = \Phi(\piv)
	\end{equation*} 
	%where
	\begin{equation*}
		p_1(\tau) = \mu_0(s_0) \cdot \prod_{t=0}^{T-1} P(s_{t+1} \vert s_t, a_t, \mu_t) \cdot \prod_{t=0}^T \pi_t(a_t \vert s_t)  
	\end{equation*}
	\begin{equation}\label{eq:trajectory_ERMFNE}
			\begin{aligned}
				p_2(\tau) \propto &~ \mu_0(s_0) \cdot \prod_{t=0}^{T-1} P(s_{t+1} \vert s_t, a_t, \mu_t) \cdot e^{\Rmc(\tau)}    
			\end{aligned}
	\end{equation}
\end{theorem}  
\begin{proof}
	See Appendix~\ref{proof:trajectory_ERMFNE}.
\end{proof}

\begin{figure}
	\centering
	\includegraphics[width=.4\textwidth]{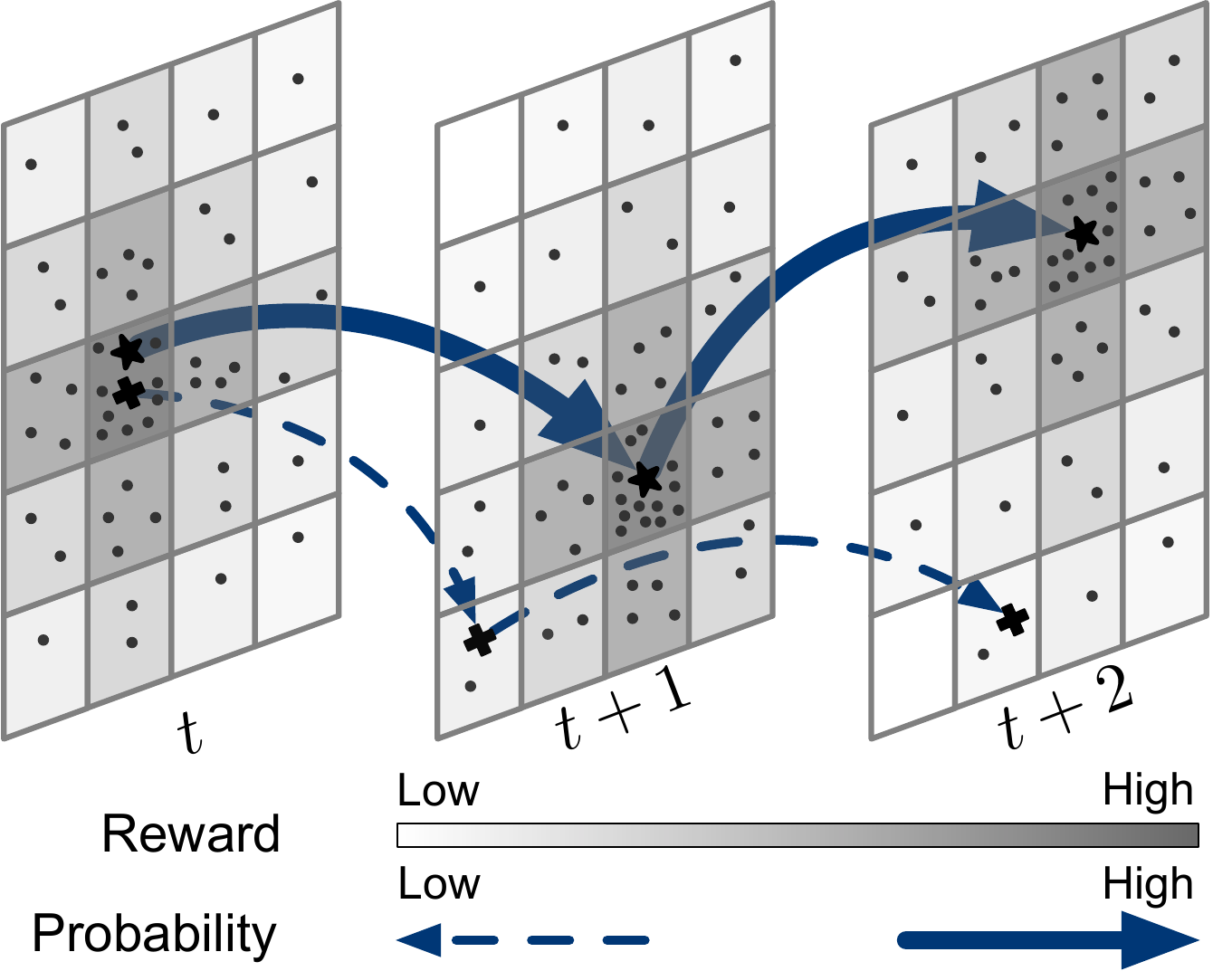}
	\caption{Illustration for the trajectory distribution induced by ERMFNE. A group of agents (dots) drift in a grid world over three time steps. Each grid represents a specific state. The distribution of agents represents the mean field. A darker grey background denotes a higher reward associated with a state. Two individual agents marked with ``$\star$'' and ``$\bm{+}$'' start at the same state. The solid arrows denote the trajectory of $\star$, which can be observed with an exponentially higher probability than the trajectory of $\bm{+}$ depicted using dashed arrows. Note that the rewards of states change over time because of the evolution of mean fields.}
	\label{fig:trajectory}
\end{figure}

\section{Extending MaxEnt IRL to MFGs}\label{sec:extension}

From now on, we assume that observed trajectories are sampled from a unique ERMFNE $(\muv^E, \piv^E)$. 
Let $r_\omega(s,a,\mu)$ be an $\omega$- parameterised reward function and $(\muv^\omega, \piv^\omega)$ denote the ERMFNE induced by $\omega$. %Also assume $(\muv^E, \piv^E)$ is induced by some {\em unknown} true parameter $\omega^*$, i.e., $(\muv^E, \piv^E) = (\muv^{\omega^*}, \piv^{\omega^*})$. 
Then, recovering the underlying reward function reduces to tuning $\omega$. The probability of a trajectory $\tau = \{(s_t, a_t)\}_{t=1}^T$ induced by ERMFNE with $r_\omega$ is defined by the following generative process:
\begin{equation}\label{eq:MLE-ERMFNE}
	p_\omega(\tau) = \mu_0(s_0) \cdot P(s_{t+1} \vert s_t, a_t, \mu_t^\omega) \cdot \prod_{t=0}^T \pi_t^\omega(a_t \vert s_t).
\end{equation}
In the spirit of MaxEnt IRL, we should tune $\omega$ by maximising the likelihood of the observed trajectories concerning the distribution defined in Eq.~\eqref{eq:MLE-ERMFNE}.
%\begin{equation}\label{eq:MLE-ERMFNE}
%	\max_\omega \Ebb_{\tau \sim \piv^E} \left[ \sum_{t=0}^T \log \pi^\omega_t(a_t \vert s_t) + \sum_{t=0}^{T-1} \log P(s_{t+1} \vert s_t, a_t, \mu^\omega_t)  \right],
%\end{equation}
%where $\mu_0$ is omitted as it is independent of $\omega$. 
By Theorem~\ref{thm:trajectory_ERMFNE}, we can instead optimise the likelihood with respect to the distribution defined in Eq.~\eqref{eq:trajectory_ERMFNE} as a variational approximation:
\begin{equation}\label{eq:MFIRL_original}
\small
\begin{aligned}
	\max_\omega  \Lmc(\omega)
	 \triangleq \Ebb_{\tau \sim (\muv^E,\piv^E)} \Big[ \Rmc_\omega(\tau) + \sum_{t=0}^{T-1}   
		  \log P(s_{t+1} \vert s_t, a_t, \mu^\omega_t)  \Big] - \log \Zmc_\omega,
\end{aligned}
\end{equation}
%taking ERMFNE as the optimal notion allows us to rationalise the expert behaviours by maximising the likelihood of demonstrated trajectories with respect to the distribution defined by Eq.~\eqref{eq:trajectory_ERMFNE}. Due to the homogeneity of agents, trajectories of all $N$ expert agents are drawn from the same maximum entropy distribution. Hence, we can tune $\omega$ by maximising the likelihood over trajectories of all $N$ expert agents, which can be reduced to the following MLE problem:
%\begin{equation}\label{eq:MFIRL_original}
%\small
%\begin{aligned}
%	\max_\omega  L(\omega) & \triangleq \mathbb{E}_{\tau_j^i \sim \piv_E} \left[\log(\Pr\nolimits_\omega(\tau_j^i))\right]\\
%	 = \frac{1}{M} & \sum_{j=1}^M  \frac{1}{N} \sum_{i=1}^N \sum_{t=0}^{T-1} \left( \gamma^t  r_\omega \left(s_{j,t}^i, a_{j,t}^i, \mu^\omega_t \right)   
%		+   \log p(s^i_{j,t+1} \vert s_{j,t}^i, a_{j,t}^i, \mu^\omega_t) \right) - \log Z_\omega,
%\end{aligned}
%\end{equation}
where $\Zmc_\omega$ is the partition function of the distribution defined in Eq.~\eqref{eq:trajectory_ERMFNE}, i.e., a summation over all feasible trajectories:
\begin{equation}\label{eq:partition_original}
\begin{aligned}
	\Zmc_\omega \triangleq &~\sum_{\tau} e^{R_\omega(\tau)} \prod_{t=0}^{T-1} \log P(s_{t+1} \vert s_t, a_t, \mu^\omega_t).
\end{aligned}
\end{equation}
%The initial mean field $\mu_0$ is omitted in Eq.~\eqref{eq:MFIRL_original} and Eq.~\eqref{eq:partition_original} as it does not depend on the reward parameter  $\omega$.

However, directly optimising the likelihood objective in Eq.~\eqref{eq:MFIRL_original} is intractable because we cannot analytically derive the MF flow $\muv^\omega$. %This problem also exists for MFNE. 
This problem arises from the nature of MFGs that the policy and MF flow in ERMFNE are interdependent because $\tilde{\piv}^\star = \tilde{\Psi}(\tilde{\muv}^\star)$ and in turn $\tilde{\muv}^\star = \Phi(\tilde{\piv}^\star)$. %(also in MFNE \citep{lasry2007mean}). %, as $\piv^* = \tilde{\Psi}(\muv^*)$ and in turn $\muv^* = \Phi(\piv^*)$. %As a result, computing a ERMFNE is analytically intractable. 
This issue poses the main challenge for extending MaxEnt IRL to MFGs. 
Worse yet, the transition function $P$ also depends on $\muv^\omega$,  posing an extra layer of complexity as the environment dynamics is generally unknown in practice.

While, notice that if we have access to an ``oracle'' (known a priori) MF flow that determines the shape of the observed MF flow $\muv^E$, an individual would be decoupled from the population. %according to the ERMFNE operator $\tilde{\Gamma} = \Phi \circ \tilde{\Psi}$. 
Inspired by this fact, we sidestep this problem by substituting $\muv^\omega$ with the empirical value of $\muv^E$, denoted by $\hat{\muv}^E \triangleq \{\mu_t^E\}_{t=0}^T$, estimated from observations $\Dmc_E =  \{ \tau_j \}_{j = 1}^M$ by calculating the occurrence frequencies of states:  
\begin{equation*}
\hat{\mu}_t^E(s) \triangleq \frac{1}{M} \sum_{j=1}^M \mathds{s} \mathds{1}_{ \{ s_{j,t}^i = s \} }.	
\end{equation*}
Since the population consistency condition in ERMFNE guarantees that the state marginal distribution of a single agent matches the mean field at each time step, $\hat{\muv}^E$ achieves an unbiased estimator of $\muv^E$. Meanwhile, by substituting $\hat{\muv}^E$ for $\muv^\omega$, the transition function $P(s_t,a_t,\hat{\mu}_t^E)$ is decoupled the from the reward parameter $\omega$ as $\hat{\muv}^E$ does not depend on $\omega$, and henceforth being omitted in the likelihood function. Finally, with this substitution, we obtain a tractable version of the original MLE objective in Eq.~\eqref{eq:MFIRL_original}:
\begin{equation}\label{eq:final}
	\max_{\omega} \hat{\Lmc}\left(\omega; \hat{\muv}^E \right) \triangleq \Ebb_{\tau \sim \Dmc_E} \left[ \hat{\Rmc}_\omega(\tau) \right]  - \log \hat{\Zmc}_\omega,
\end{equation}
which resembles the formulation of the MLE objective of MaxEnt IRL as given in Eq.~\eqref{eq:MaxEntIRL}. Here, $\hat{\Rmc}_\omega(\tau) \triangleq \sum_{t=0}^{T-1} r_\omega(s_t,a_t,\hat{\mu}_t^E)$ and $\hat{\Zmc}_{\omega} \triangleq \sum_{\tau \in \Dmc_E} e^{\hat{\Rmc}_\omega(\tau)}$ denotes the simplified partition function in Eq.~\eqref{eq:partition_original}.

%Intuitively, Eq.~\eqref{eq:final} can be interpreted as that we maximise the likelihood of expert trajectories with respect to the trajectory distribution induced by the best-response policy to $\hat{\muv}^E$. 
Statistically, Eq.~\eqref{eq:final} can be interpreted as that we use a likelihood function of a ``mis-specified'' model that treats the policy and MF flow as being independent and replaces the MF flow with its empirical value. In this manner, we estimate the optimal solution to the original MLE problem by maximising a simplified form of the actual likelihood function defined in Eq.~\eqref{eq:MFIRL_original}. Although we sacrifice the accuracy for achieving tractability due to the estimation error of $\hat{\muv}^E$, we show that Eq.~\eqref{eq:final} preserves the property of the asymptotic consistency, %that the property of the asymptotic consistency persists in Eq.~\eqref{eq:final} , 
as $\hat{\muv}^E$ converges almost surely to $\muv^E$ as the number of samples tends to infinity due to the law of large numbers.

\begin{theorem}\label{thm:MFIRL}\label{thm:consistent}
Let the observed trajectories in $\Dmc_E =  \{ \tau_j \}_{j = 1}^M$ be independent and identically distributed and sampled from a unique ERMFNE induced by an unknown parameterised reward function. Suppose for all $s \in \Smc$, $a \in \Amc$ and $\mu \in \Pmc(\Smc)$, $r_\omega(s,a,\mu)$ is differentiable w.r.t. $\omega$. Then, 
	with probability $1$ as the number of samples $M \to \infty$, the equation $\nabla_\omega \hat{\Lmc}\left(\omega; \hat{\muv}^E \right) = 0$ has a root $\hat{\omega}$ such that $\hat{\omega}$ is a maximiser of the likelihood objective $\Lmc(\omega)$ in Eq.~\eqref{eq:MFIRL_original}.
\end{theorem}
\begin{proof}
	%See Appendix~\ref{proof:consistency}.
	The gradient of $\hat{\Lmc}$ concerning $\omega$ is given by:
	\begin{equation}\label{eq:derivative}
	\begin{aligned}
		\nabla_\omega \hat{\Lmc}\left(\omega; \hat{\muv}^E \right)
		 = & \frac{1}{M} \sum_{j=1}^M     \nabla_\omega \hat{\Rmc}_\omega(\tau_j)   - \nabla_\omega \log \hat{\Zmc}_\omega\\
		 = & \frac{1}{M} \sum_{j=1}^M   \nabla_\omega \hat{\Rmc}_\omega(\tau_j)   -  \frac{1}{\hat{\Zmc}_\omega} \nabla_\omega \hat{\Zmc}_\omega\\
		= &  \frac{1}{M} \sum_{j=1}^M  \nabla_\omega \hat{\Rmc}_\omega(\tau_j) - \sum_{j=1}^M \frac{e^{\hat{\Rmc}_\omega(\tau_j)}}{\hat{\Zmc}_\omega}   \nabla_\omega \hat{\Rmc}_\omega(\tau_j). 
		%\frac{1}{MN} \sum_{j=1}^M \sum_{i=1}^N \sum_{t=0}^{T-1} \gamma^t  \nabla_\omega r_\omega(s_{j,t}^i, a_{j,t}^i, \hat{\mu}^E_t)  -  \sum_{j=1}^M \sum_{i=1}^N \Pr(\tau_j) \sum_{t=0}^{T-1} \gamma^t \nabla_\omega r_\omega(s_{j,t}^i, a_{j,t}^i, \hat{\mu}^E_t)\\
		\end{aligned}
	\end{equation}
	
Let $\Pr\nolimits_{\Dmc_E}(\tau) \triangleq \frac{1}{M} \sum_{j=1}^M  \mathds{1}_{\{\tau_j = \tau\}}$ denote the empirical trajectory distribution, then Eq.~\eqref{eq:derivative} is equivalent to 
\begin{equation}\label{eq:derivative2}
\begin{aligned}
	 & \nabla_\omega \hat{\Lmc}\left(\omega; \hat{\muv}^E \right)  =    \sum_{j=1}^M \Pr\nolimits_{\Dmc_E}(\tau_j) \nabla_\omega \hat{\Rmc}_\omega(\tau_j) - 
	  \sum_{j=1}^M  \frac{e^{\hat{\Rmc}_\omega(\tau_j)}}{\hat{\Zmc}_\omega}  \nabla_\omega \hat{\Rmc}_\omega(\tau_j) \\
\end{aligned}
\end{equation}
\begin{equation*}
	\qquad\;\; =  \sum_{j=1}^M \left( \Pr\nolimits_{\Dmc_E}(\tau_j) - \frac{e^{\hat{\Rmc}_\omega(\tau_j)}}{\hat{\Zmc}_\omega} \right)   \nabla_\omega \hat{\Rmc}_\omega(\tau_j).
\end{equation*}

When the number of samples $M \to \infty$, $\Pr\nolimits_{\Dmc_E}(\tau)$ tends to the true trajectory distribution $p(\tau)$ (see Eq.~\eqref{eq:trajectory_ERMFNE}) induced by ERMFNE. 
Meanwhile, according to the law of large numbers, $\hat{\muv} \to \muv^E$ with probability one as the number of samples $M \to \infty$. Let $\omega^\star$ be a maximiser of the likelihood objective in Eq.~\eqref{eq:MFIRL_original}. Taking the limit as $M \to \infty$ and the optimality $\omega = \omega^\star$, we have:
\begin{equation*}
\begin{aligned}
	\frac{e^{\Rmc_{\omega^\star}(\tau_j)}}{\hat{\Zmc}_{\omega^\star}}  & = \frac{e^{\Rmc_{\omega^\star}(\tau_j)}}{\sum_{j=1}^M e^{\Rmc_{\omega^\star}(\tau_j) }}
	 = \Pr(\tau_j)
	 = \Pr\nolimits_{\Dmc_E}(\tau_j).
\end{aligned}
\end{equation*}
Therefore, the gradient in Eq.~\eqref{eq:derivative2} will be zero. %; hence, the proof is complete.
\end{proof}

\section{Mean Field Adversarial IRL}\label{sec:MF-AIRL}
Theorem~\ref{thm:MFIRL} bridges the gap between optimising the original intractable MLE objective in Eq.~\eqref{eq:MFIRL_original} and the tractable empirical MLE objective in Eq.~\eqref{eq:final}. However, as mentioned in Sec.~\ref{sec:MaxEnt}, exactly computing the partition function $\hat{Z}_\omega$ is generally difficult. Similar to AIRL \citep{fu2018learning}, 
we adopt {\em importance sampling} to estimate $\hat{Z}_{\omega}$ with {\em adaptive samplers}. Since policies are time-varying in MFGs, we use a set of $T$ adaptive samplers $\piv^{\thetav} \triangleq (\pi^{\theta_0}, \pi^{\theta_1}, \ldots, \pi^{\theta_{T-1}})$, where each $\pi^{\theta_t}$ serves as the parameterised per-step policy.

Now, we are ready to present to our {\em Mean-Field Adversarial IRL} (MF-AIRL) framework, which trains a discriminator $$\hat{D}_\omega (s_t, a_t) \triangleq \frac{e^{f_\omega(s_t, a_t, \hat{\mu}^E_t)}}{e^{f_\omega(s_t, a_t, \hat{\mu}^E_t) + \pi^{\theta_t}(a_t \vert s_t)}}$$ as
\begin{equation}\label{eq:MFAIRL-discriminator}
\begin{aligned}
	\max_\omega~ &\Ebb_{\tau \sim \Dmc_E} \left[ \sum_{t=0}^{T-1} \log \hat{D}_\omega (s_t, a_t) \right] + \Ebb_{\tau \sim \piv^{\thetav}} \left[ \sum_{t=0}^{T-1} \log (1 - \hat{D}_\omega (s_t, a_t)) \right],
\end{aligned}
\end{equation}
and trains adaptive importance samplers $\piv^{\thetav}$ as
\begin{equation}\label{eq:MFAIRL-samplers}
\begin{aligned}
	\max_{\thetav} &\Ebb_{\tau \sim \piv^{\thetav}} \left[ \sum_{t=0}^{T-1} \log \hat{D}_\omega (s_t, a_t) - \log (1 - \hat{D}_\omega (s_t, a_t)) \right] \\
	= &\Ebb_{\tau \sim \piv^{\thetav}} \left[ \sum_{t=0}^{T-1} f_\omega(s_t, a_t, \hat{\mu}^E_t) - \log \pi^{\theta_t} (a_t \vert s_t) \right].
\end{aligned}
\end{equation}

The update of policy parameters $\thetav$ is interleaved with the update of the reward parameter $\omega$. Intuitively, tuning $\piv^{\thetav}$ can be viewed as a {\em policy optimisation} procedure, which is to find the ERMFNE policy induced by the current reward parameter in order to minimise the variance of importance sampling; $f_\omega$ is trained to estimate the reward function by distinguishing between the observed trajectories and those generated by the current adaptive samplers $\piv^{\thetav}$. We can use {\em backward induction} to train $\piv^{\thetav}$, i.e., tuning $\pi^{\theta_t}$ based on $\pi^{\theta_{t+1}}, \ldots, \pi^{\theta_{T-1}}$ that are already tuned. \footnote{Since the reward at $t = T$ is 0, $\pi_T$ always selects actions with ties broken arbitrarily to maximise the policy entropy.} At optimality, $f_\omega$ will approximate the underlying reward function for the observed ERMFNE and $\piv^{\thetav}$ will approximate the observed policy.

\subsection{Reward Shaping in MFGs}\label{sec:shaping}
As mentioned in Sec.~\ref{sec:MaxEnt}, IRL faces reward ambiguity. This issue is called the effect of {\em reward shaping} \citep{ng1999policy}, i.e., there is a class of reward transformations that induce the same set of optimal policies, where IRL cannot identify the ground-truth one without prior knowledge of environments. %Moreover, when such a transformation involves the environment dynamics, the policy (equilibrium) elicited by the learned reward function may no longer be an unbiased estimation of the ground-truth one, if the environment dynamics changes \cite{fu2018learning}.  
It is shown that for any state-only {\em potential function} $h: \Smc \to \mathbb{R}$, the reward transformation %({\em potential-based reward shaping}) 
$$r'(s_t, a_t, s_{t+1}) = r(s_t, a_t) + h(s_{t+1}) - h(s_t)$$ 
%\begin{equation*}
%	r'(s_t, a_t, s_{t+1}) = r(s_t, a_t) + \gamma f(s_{t+1}) - f(s_t)
%\end{equation*}
is the sufficient and necessary condition to ensure policy invariance for both MDPs and stochastic games \citep{devlin2011theoretical}. We show that a similar idea can be extended to MFGs: for any potential function $g: \Smc \times \Pmc(\Smc) \to \Rbb$, the potential-based reward shaping can ensure the invariance of both ERMFNE and MFNE. The detailed justification and proofs are deferred until Appendix~\ref{app:shaping}.

To mitigate the effect of reward shaping, similar to AIRL \citep{fu2018learning}, we assume that the parameterised reward function $f_\omega$ is in the following structure:
\begin{equation*}
\begin{aligned}
	f_{\omega, \phi}(s_t, a_t, \mu_t, s_{t+1}, \mu_{t+1}) = ~& r_\omega(s_t, a_t, \mu_t) +\\  & g_\phi(s_{t+1}, \mu_{t+1}) - g_\phi(s_t, \mu_t).
\end{aligned}
\end{equation*}
Here, $g_\phi$ is the $\phi$-parameterised potential function for MFGs. To summarise, we present the pseudocode in Alg.~\ref{alg:MFIRL}.

\begin{algorithm}[htp]
   \caption{Mean-Field Adversarial IRL}\label{alg:MFIRL}
\begin{algorithmic}[1]
   \STATE {\bf Input:} MFG with parameters $(\Smc, \Amc, P, \mu_0)$ and observed trajectories $\Dmc_E = \{ \tau_j \}_{j = 1}^M$.
   \STATE {\bf Initialisation:} reward parameter $\omega$, adaptive samplers $\thetav$ and potential function parameter $\phi$.
   \STATE Estimate the empirical expert MF flow $\hat{\muv}^E$ from $\Dmc_E$. % according to Eq.~\eqref{eq:est_mf}.
   \FOR{each iteration}
   		\STATE Sample a set of trajectories $\Dmc_{\piv} = \{\tau_j\}$ from $\piv^\thetav$ via $s_0 \sim \mu_0$, $a_t \sim \pi^{\theta_t}(\cdot \vert s_t)$, $s_{t+1} \sim P(\cdot \vert s_t, a_t, \mu_t)$.
   		\STATE Sample subsets $\Xmc_E,\Xmc_{\piv}$ from $\Dmc_E,\Dmc_{\piv}$.
   		\STATE Update $\omega, \phi$ using $\Xmc_E,\Xmc_{\piv}$ according to Eq.~\eqref{eq:MFAIRL-discriminator}.
   		\FOR{$t=T-1, \ldots, 0$}
   			\STATE Update $\theta_t$ with respect to the reward estimate $r_\omega(s,a,\mu) + g_\phi(s,\mu)$ according to Eq.~\eqref{eq:MFAIRL-samplers}.
   		\ENDFOR 
   \ENDFOR
   \STATE {\bfseries Output:} Learned reward function $r_{\omega}$.
\end{algorithmic}
\end{algorithm}

\section{Related Work}\label{sec:related}
We continue from the introduction to relate our work to the existing literature.
MFGs were pioneered %through a series of work such as
by \citep{lasry2007mean,huang2006large} in the continuous setting of stochastic differential games. %Mathematically, the dynamics of the system is governed by two stochastic differential equations: the Hamilton-Jacobi-Bellman equation models the backward dynamics of a representative agent's value functions and the Fokker-Planck equation models the forward dynamics of mean fields.
The discrete MFG model was then proposed in \citep{gomes2010discrete}, which was adopted in the learning setting. %We adopt the same definition as in \cite{elie2020convergence} (See Section~\ref{sec:problem}). %\cite{gueant2011mean} provides a survey of MFG models and discussed various applications. 
%Pioneering works of \cite{lasry2007mean} and \cite{huang2006large} originated MFG %each independently proposing MFG for modelling large-scale problems in economics and analyse the existence and uniqueness of solutions. They address MFG in the continuous setting. Mathematically, the dynamics of an MFG is governed by two stochastic differential equations: the {\em Hamilton-Jacobi-Bellman} (HJB) equation models the backward dynamics of a representative agent's value function, and the {\em Fokker–Planck} equation models the forward dynamics of the MF flow. We define the action value function for MFG by absorbing actions into discrete HJB equations. Recently, models of discrete MFG are proposed in \cite{gomes2010discrete,saldi2018markov}, which inform the choice of the discrete setting in our work.  %\cite{gueant2011mean} provided a survey of MFG models and discussed various applications in continuous time and space, such as a model of population distribution that informed the choice of application in our work. Even though the MFG framework is agnostic towards the choice of the cost function (i.e. negative reward), prior work makes strong assumptions on the cost in order to attain analytic solutions. We take a contrasting view that the dynamics of any game is heavily impacted by the reward function, and hence we propose methods to learn the MFG reward function from demonstrations.
Recently, learning MFG has attracted significant attention \citep{cardaliaguet2017learning}, and most methods are based on reinforcement learning \citep{yang2018mean,guo2019learning,subramanian2019reinforcement,cui2021approximately}, {\em fictitious play} \citep{cardaliaguet2017learning,elie2020convergence,xie2021learning}, or a combination of the two \citep{perrin2022generalization}.    
While these methods require a well-designed reward function that is challenging to hand-tune in practice. In contrast, our method recovers a reward function from the observed behaviour.

IRL was introduced by \citep{ng2000algorithms} in the single-agent setting. Early IRL methods were based on {\em margin optimisation} \citep{ratliff2006maximum}, which makes IRL ill-defined. {\em Maximum entropy} (MaxEnt) {\em IRL} \citep{ziebart2008maximum,ziebart2010modeling} addresses this issue by providing a probabilistic approach to find a most non-committal reward function whose induced state-action trajectory distribution has the MaxEnt among those matching the reward expectation of the observed behaviour. %Most critically, the MaxEnt principle enables characterising uncertainties caused by the bounded rationality of experts. 
However, it is only suitable for small and discrete problems since MaxEnt IRL requires iteratively solving an RL problem while tuning a reward function. 
{\em Adversarial IRL} (AIRL) \citep{fu2018learning} was later proposed, which scales MaxEnt IRL to high-dimensional or continuous domains. It implements a sampling-based approximation to MaxEnt IRL by drawing an analogy between {\em generative adversarial networks} \citep{goodfellow2014generative} and MaxEnt IRL, thereby being able to partially solve each RL problem associated with reward tuning. 

Another line of work extends ILR to the multi-agent setting, where the problem is cast to finding individual reward functions of stochastic games. They typically take a specific equilibrium concept, such as the conventional Nash equilibrium \citep{fu2021evaluating}, logistic stochastic best response equilibrium \citep{yu2019multi} and equilibrium for the cooperative setting \citep{natarajan2010multi}, and assume the equilibrium exists uniquely in order to guarantee the well-definedness (as we assumed in this paper). However, these methods scale poorly to large-scale scenarios due to the exponential growth of state-action spaces and agent interactions. By extending MaxEnt IRL to MFGs, our method realises an effective IRL framework for large-scale scenarios. %, which is able to reason about uncertainties in observed behaviour. Through integrating this extension and AIRL, we further develop a practical framework implementation.

\begin{figure*}
	\centering
	\includegraphics[width= \textwidth]{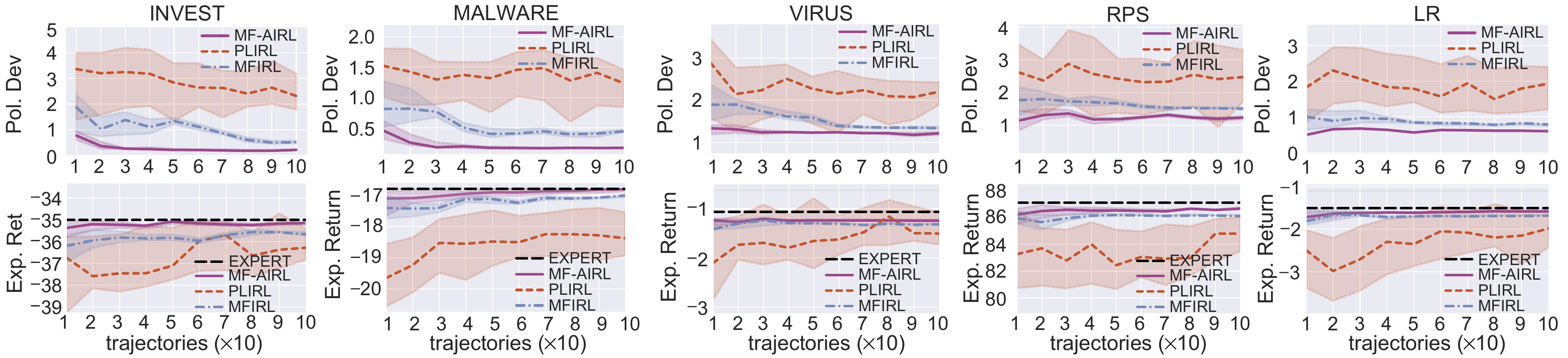}
	\caption{Results for numerical tasks. The line and shade are the median and variance over 10 independent runs.}\label{fig:numerical}
\end{figure*}

\section{Experiments}\label{sec:experiments}
We seek to answer the following fundamental question via experiments: {\em Can MF-AIRL effectively and efficiently recover a suitable reward function of an MFG by observing bounded rational behaviour?} To that end, we evaluate MF-AIRL on a series of simulated tasks motivated by real-world applications, where the observed behaviour is sampled from an ERMFNE.

\subsection{Experimental Setup}

\subsubsection{Tasks.} We adopt five MFG tasks: {\em investment in  product quality} ({\tt INVEST} for short), {\em malware spread} ({\tt MALWARE}), {\em virus infection} ({\tt VIRUS}), {\em Rock-Paper-Scissors} ({\tt RPS}) and {\em Left-Right} ({\tt LR}), which simulate a series of large-scale multi-agent scenarios in the contexts of marketing, virus propagation modelling and norm emergence. These tasks were originally studied in \citep{weintraub2010computational,huang2016mean,huang2017mean,subramanian2019reinforcement,cui2021approximately} and adapted by \citep{chen2022individual}. %Among five tasks, {\tt VIRUS} and {\tt LR} are set as fully cooperative scenarios where the expert ERMFNE maximises the population's average rewards; while the other three tasks are competitive. 
Detailed descriptions and settings are	 in Appendix~\ref{app:task}.

\subsubsection{Baselines.} We compare MF-AIRL against the two IRL methods above for MFGs: 
{\bf (1)} 	The centralised method \citep{yang2018learning} relies on the reduction from MFG to MDP. Since it aims to recover the population's average rewards, we call it {\em population-level IRL} ({\bf PLIRL}). Shown by \citep{chen2022individual}, PLIRL is only compatible with socially optimal equilibria that maximise the population's average rewards; otherwise, it can result in biased reward inference.
{\bf (2)} The decentralised method, {\em Mean Field IRL} ({\bf MFIRL}) \citep{chen2022individual}, is based on margin optimisation, i.e., finding a reward function by minimising the margin (in terms of expected return) between the observed equilibrium and every other equilibrium; it is able to recover the ground-truth reward function with no bias, regardless of whether the observed equilibrium is socially optimal or not. 
%Note that both methods are based on MFNE.

\subsubsection{Performance Metrics.} The quality of a learned reward function $r_\omega$ can be evaluated by the difference between its induced best-response policy to $\muv^E$, denoted by $\piv^\omega$, and $\piv^E$, because a best-response policy is unique under the entropy regularisation.  We adopt the following two metrics that measure the difference between $\piv^\omega$ and $\piv^E$ reflected in the statistical distance and the expected return, respectively: 

\begin{enumerate}
	\item {\em Policy Deviation} (Pol. Dev). We use the cumulative KL- divergence,  $$\sum_{t = 0}^{T-1} \sum_{s \in \Smc} D_{\KL} \big( \pi^E_t (\cdot \vert s) \parallel \pi^\omega_t (\cdot \vert s) \big),$$ to measure the statistical distance between two policies.
    %\item {\em MF flow Deviation} (Dev. MF). We use the cumulative KL-divergence, $\sum_{t = 0}^{T} D_{\KL} \big(  \tilde{\mu}^\star_t \parallel \mu^\omega_t \big)$, to measure the difference over two MF flows.
    \item {\em Expected return} (Exp. Ret). The difference between two expected returns $J(\muv^E, \piv^\omega)$ and $J(\muv^E, \piv^E)$ under the ground-truth reward function. 
\end{enumerate}

\vspace{-1em}
\subsubsection{Training Procedures.} In all tasks, we have access to the ground-truth reward functions and environment dynamics, which allows us to numerically compute an ERMFNE through the fixed point iteration as introduced in Sec.~\ref{sec:ERMFNE}. Unless specified otherwise, we set the entropy regularisation coefficient $\beta=1$. After obtaining the ERMFNE, we sample trajectories from them, each with a length of 50 time steps, the same as the number used in \citep{song2018multi,yu2019multi,chen2022individual}. We use one-hot encoding to represent states and actions. All three algorithms share the same neural network architecture as the reward model: two hidden layers of 64 leaky rectified linear units each.
Implementation details are given in Appendix~\ref{app:exp}.

\begin{table}
\begin{threeparttable}
    \caption{Results for new environment dynamics. Mean and variance are taken across 10 independent runs.}
    \label{tab:new}
    \centering 
    \small
    \begin{tabular}{ccrrr}
        \toprule
        \multirow{3}{*}{Task} & \multirow{3}{*}{Metric} & \multicolumn{3}{c}{Algorithm}\\
        \cmidrule(r){3-5}
        & &   MF-AIRL & PLIRL & MFIRL \\
        \midrule
        \multirow{2}{*}{{\tt INVEST}}
        & Pol. Dev  &  \textbf{0.24} (0.02) & 1.06 (0.21) & 0.78 (0.18)\\
        & Exp. Ret  &   \textbf{-35.19} (0.51) & -37.73 (2.76) & -35.92 (0.98)\\
        \midrule
        \multirow{2}{*}{{\tt MALWARE}} 
        & Pol. Dev &  \textbf{0.52} (0.01) & 1.54 (1.20) & 0.73 (0.14)\\
        & Exp. Ret &  \textbf{-18.49} (0.14) & -19.59 (0.29) & -18.82 (0.05) \\
        \midrule
        \multirow{2}{*}{{\tt VIRUS}} 
        & Pol. Dev &  \textbf{1.48} (0.01) & 1.76 (0.18) & 1.55 (0.03)\\
        & Exp. Ret &  \textbf{-1.71} (0.02) & -2.66 (0.14) & -2.16 (0.06)\\
        \midrule
        \multirow{2}{*}{{\tt RPS}} 
        & Pol. Dev & \textbf{6.11} (0.46) & 6.47 (0.98) & 6.56 (0.82)\\
        & Exp. Ret & \textbf{93.36} (2.51) & 91.99 (0.44) & 91.28 (2.15)\\
        \midrule
        \multirow{2}{*}{{\tt LR}} 
        & Pol. Dev &  \textbf{0.57} (0.04) & 0.62 (0.22) & 0.71 (0.07)\\
        & Exp. Ret &  \textbf{-1.70} (0.01) & -2.67 (1.01) & -1.93 (0.06)\\
        \bottomrule
    \end{tabular}
     \begin{tablenotes}
      \small
      \item {\em Note:} The Exp. Ret of the observed behaviour for five tasks are -35.87, -18.90, -1.24, 93.16 and -0.64, respectively.
    \end{tablenotes}
  \end{threeparttable}
\end{table}

\subsection{Reward Recovery with Fixed  Dynamics} The first experiment tests the capability of reward recovery with fixed environment dynamics. Results are depicted in Fig.~\ref{fig:numerical}. On all tasks, MF-AIRL achieves the closest performance to the observed behaviour with the same number of trajectories and the fastest convergence with the number of trajectories increases, suggesting that among all three algorithms, MF-AIRL is the most effective and efficient for bounded-rational agents. MFIRL shows larger deviations even if the number of trajectories is large. This may be because MFIRL takes MFNE as the solution concept, thereby lacking the ability to tolerate suboptimal behaviours. PLIRL shows the largest deviation and variance. This is as expected because PLIRL is only suitable for socially optimal equilibria, while an ERMFNE is not socially optimal as it captures bounded rationality. Therefore, biased reward inferences occur when applying PLIRL to these tasks.

\subsection{Policy Transfer across Varying Dynamics}
The second experiment investigates the robustness against changing environment dynamics. We change the transition function (see Appendix~\ref{app:task} for details), recompute an ERMFNE $(\muv^E_{new},\piv^E_{new})$ induced by the ground-truth reward function, compute the best-response policy $\piv^\omega_{new}$ to $\muv^E_{new}$ under the learned reward function (trained with 100 demonstrated trajectories), and calculate two metrics again. Results are summarised in Tab.~\ref{tab:new}. Consistently, MF-AIRL outperforms two baselines on all tasks. We attribute the high robustness of MF-AIRL to the following reasons: (1) MF-AIRL uses a potential function to mitigate the effect of reward shaping while two baselines do not; (2) The issue of biased inference in PLIRL can be exacerbated by the changing dynamics, as is argued in \cite{chen2022individual}. To summarise, MF-AIRL can recover ground-truth reward functions with high robustness to changing dynamics.

\begin{table}
\begin{threeparttable}
\caption{Comparisons between MF-AIRL and MFIRL on varying entropy regularisation strength $\beta$.}
    \label{tab:beta}
    \centering 
    \small
    \begin{tabular}{clcrrr}
        \toprule
        \multirow{3}{*}{Task} & \multirow{3}{*}{$\beta$}& \multirow{3}{*}{Metric} & \multicolumn{3}{c}{Algorithm}\\
        \cmidrule(r){4-6}
        & & & OBSERVED & MF-AIRL & MFIRL \\
        \midrule
        \multirow{5}{*}{{\tt INVEST}} &
        \multirow{2}{*}{$0$}
        &  Pol. Dev  & --  & 0.45 (0.03) & 0.44 (\textbf{0.02})\\
        & & Exp. Ret  & -35.05  & -35.92  (\textbf{0.68}) & -35.54 (2.55)\\
        \cmidrule(r){2-6}
        & \multirow{2}{*}{$0.1$}
        &  Pol. Dev  & --& \textbf{0.31} (0.02)  & 0.39 (0.04)\\
        & & Exp. Ret  & -36.37  & \textbf{-37.08} (0.71) & -37.40 (1.08)\\
        \midrule
        \multirow{5}{*}{{\tt MALWARE}} &
        \multirow{2}{*}{$0$}
        &  Pol. Dev  & --  & 0.39 (0.07) & 0.38 (0.07)\\
        & & Exp. Ret  & -18.06 &  -19.15 (\textbf{0.25}) & -18.52 (0.51)\\
        \cmidrule(r){2-6}
        & \multirow{2}{*}{$0.1$}
        &  Pol. Dev  & -- & \textbf{0.41} (0.03)  & 0.50 (0.10)\\
        & & Exp. Ret  & -19.36 & \textbf{-19.84} (0.22)  & -20.39 (0.69)\\
        \midrule
        \multirow{5}{*}{{\tt VIRUS}} &
        \multirow{2}{*}{$0$}
        &  Pol. Dev  & -- & 1.50 (\textbf{0.04})  & 1.34 (0.09)\\
        & & Exp. Ret  & -1.17  & -2.55 (\textbf{0.04}) & -1.61 (0.17)\\
        \cmidrule(r){2-6}
        & \multirow{2}{*}{$0.1$}
        &  Pol. Dev  & -- & \textbf{1.54}  (0.01)  & 1.80 (0.07)\\
        & & Exp. Ret  & -2.15  & \textbf{-2.61} (0.06) & -2.98 (0.43)\\
        \midrule
        \multirow{5}{*}{{\tt RPS}} &
        \multirow{2}{*}{$0$}
        &  Pol. Dev  & -- & 9.71 (\textbf{0.24})  & 9.36 (0.40)\\
        & & Exp. Ret  & 94.27  & 93.21 (\textbf{0.49}) & 93.58 (2.51)\\
        \cmidrule(r){2-6}
        & \multirow{2}{*}{$0.1$}
        &  Pol. Dev  & -- & \textbf{7.09} (0.54)  & 8.40 (0.40)\\
        & & Exp. Ret  & 91.43  & \textbf{90.43} (3.09) & 89.40 (0.96)\\
        \midrule
        \multirow{5}{*}{{\tt LR}} &
        \multirow{2}{*}{$0$}
        &  Pol. Dev  & -- & 0.45 (\textbf{0.07}) & 0.37 (0.08)\\
        & & Exp. Ret  & -0.52  & -2.60 (\textbf{0.08}) & -1.70 (1.08)\\
        \cmidrule(r){2-6}
        & \multirow{2}{*}{$0.1$}
        &  Pol. Dev  & -- & \textbf{0.68} (0.04)  & 0.70 (0.04)\\
        & & Exp. Ret  & -0.64  & \textbf{-0.81} (0.04) & -0.96 (0.71)\\
        \bottomrule
    \end{tabular}
\end{threeparttable}
\end{table}

\subsection{Weak Entropy Regularisation}
Suppose the entropy regularisation in ERMFNE is too strong. In that case, it becomes easy and trivial to perform MaxEnt IRL as the policy tends to select actions uniformly due to the maximum entropy principle. Our third experiment thus investigates the performance under weak entropy regularisation. Since PLIRL is known to lead to biased reward inference if the demonstrated equilibrium is not socially optimal, we eliminate it here and only compare two decentralised methods. To weaken the entropy regularisation, we set the coefficient $\beta$ in the observed ERMFNE to $0$ and $0.1$, respectively and sample 100 demonstrated trajectories from each. The environment dynamic is fixed. Results are summarised in Tab.~\ref{tab:beta}. Note that an ERMFNE is recovered to an MFNE if $\beta=0$; this complements the above two experiments where all trajectories are sampled from an ERMFNE with $\beta =1$. It also enables a fair comparison between MF-AIRL and MFIRL as, technically, they are designed under two prescribed equilibrium concepts.

With trajectories sampled from an MFNE ($\beta=0$), MF-AIRL shows a more significant deviation from the observed behaviour than MFIRL, but its variance is lower on average. This fact can be attributed to the reason that a best-response policy in MFNE does not exist uniquely, though MFIRL is unbiased under  MFNE. In contrast, MF-AIRL can always recover a unique best-response policy from a learned reward function, though with bias under MFNE. This result again validates our argument that by taking MFNE as the solution concept, MFIRL fails to elicit a unique policy from the learned reward function.
While, even with a positive yet small strength of entropy regularisation   ($\beta = 0.1$), our MF-AIRL quickly outperforms MFIRL in terms of both accuracy and variance. This suggests that our MF-AIRL can handle imperfect behaviours resulting from bounded rationality, even with minor uncertainties.

\section{Concluding Remarks}

In this paper, we propose MF-AIRL, the first probabilistic IRL framework effective for MFGs with bounded-rational agents. We first extend MaxEnt IRL to MFGs based on the solution concept termed ERMFNE, which allows us to characterise uncertainties in observed behaviour using the maximum entropy principle. We then develop the practical MF-AIRL framework using an adversarial learning approach to solve  MaxEnt IRL for MFGs efficiently. Experimental results on simulated tasks demonstrate the effectiveness and efficiency of MF-AIRL against existing IRL methods for MFGs.

We argue that MF-AIRL is worth following generalisations: 
%\begin{enumerate}
	{\bf (1)} {\em Continuous states and actions}. %According to \citep{cui2021approximately},  the existence and uniqueness of ERMFNE are guaranteed under standard conditions that both reward and transition functions are continuous. Therefore, 
	The arguments in this paper still hold for continuous state-action spaces, except that some techniques (e.g., $\epsilon$-net \citep{guo2019learning}) are needed to
discretise a mean field because it turns to a probability density function if states are continuous. 
	{\bf (2)} {\em Infinite time horizon}. When the time horizon tends to infinity, the mean field is shown to converge almost surely to a constant limit, resulting in the {\em stationary MFNE} \citep{subramanian2019reinforcement}. MF-AIRL is compatible with infinite time horizons because non-stationary equilibria recover stationary ones as special cases.
	{\bf (3)} {\em Generalised mean fields}. Some work \citep{guo2019learning} generalises the mean field $\mu \in \Pmc(\Smc)$ to $(\mu, \alpha) \in \Pmc(\Smc \times \Amc)$  by additionally considering population's average action $\alpha \in \Pmc(\Amc)$. MF-AIRL is adaptive to generalised mean fields by incorporating the marginal distribution $\alpha$ in all mean field arguments.
	{\bf (4)}  {\em Heterogeneous agents}. A large-scale heterogeneous multi-agent system can be converted to a homogeneous system by considering the type of the agent as a component of states \citep{ganapathi2020multi}. Our MF-AIRL is, therefore, compatible with heterogeneous agents.
%\end{enumerate}
%%%%%%%%%%%%%%%%%%%%%%%%%%%%%%%%%%%%%%%%%%%%%%%%%%%%%%%%%%%%%%%%%%%%%%%%

%%% The next two lines define, first, the bibliography style to be 
%%% applied, and, second, the bibliography file to be used.

\bibliographystyle{ACM-Reference-Format}
\balance
\bibliography{bib.bib}


\begin{thebibliography}{47}

%%% the Web address.

\ifx \showCODEN    \undefined \def \showCODEN     #1{\unskip}     \fi
\ifx \showDOI      \undefined \def \showDOI       #1{#1}\fi
\ifx \showISBNx    \undefined \def \showISBNx     #1{\unskip}     \fi
\ifx \showISBNxiii \undefined \def \showISBNxiii  #1{\unskip}     \fi
\ifx \showISSN     \undefined \def \showISSN      #1{\unskip}     \fi
\ifx \showLCCN     \undefined \def \showLCCN      #1{\unskip}     \fi
\ifx \shownote     \undefined \def \shownote      #1{#1}          \fi
\ifx \showarticletitle \undefined \def \showarticletitle #1{#1}   \fi
\ifx \showURL      \undefined \def \showURL       {\relax}        \fi
% invisible to TeX
\providecommand\bibfield[2]{#2}
\providecommand\bibinfo[2]{#2}
\providecommand\natexlab[1]{#1}
\providecommand\showeprint[2][]{arXiv:#2}

\bibitem[\protect\citeauthoryear{Anahtarci, Kariksiz, and Saldi}{Anahtarci
  et~al\mbox{.}}{2020}]%
        {anahtarci2020q}
\bibfield{author}{\bibinfo{person}{Berkay Anahtarci}, \bibinfo{person}{Can~Deha
  Kariksiz}, {and} \bibinfo{person}{Naci Saldi}.}
  \bibinfo{year}{2020}\natexlab{}.
\newblock \showarticletitle{Q-learning in regularized mean-field games}.
\newblock \bibinfo{journal}{\emph{arXiv preprint arXiv:2003.12151}}
  (\bibinfo{year}{2020}).
\newblock


\bibitem[\protect\citeauthoryear{Cardaliaguet and Hadikhanloo}{Cardaliaguet and
  Hadikhanloo}{2017}]%
        {cardaliaguet2017learning}
\bibfield{author}{\bibinfo{person}{Pierre Cardaliaguet} {and}
  \bibinfo{person}{Saeed Hadikhanloo}.} \bibinfo{year}{2017}\natexlab{}.
\newblock \showarticletitle{Learning in mean field games: the fictitious play}.
\newblock \bibinfo{journal}{\emph{ESAIM: Control, Optimisation and Calculus of
  Variations}} \bibinfo{volume}{23}, \bibinfo{number}{2}
  (\bibinfo{year}{2017}), \bibinfo{pages}{569--591}.
\newblock


\bibitem[\protect\citeauthoryear{Carmona, Delarue, and Lachapelle}{Carmona
  et~al\mbox{.}}{2013}]%
        {carmona2013control}
\bibfield{author}{\bibinfo{person}{Ren{\'e} Carmona},
  \bibinfo{person}{Fran{\c{c}}ois Delarue}, {and} \bibinfo{person}{Aim{\'e}
  Lachapelle}.} \bibinfo{year}{2013}\natexlab{}.
\newblock \showarticletitle{Control of McKean--Vlasov dynamics versus mean
  field games}.
\newblock \bibinfo{journal}{\emph{Mathematics and Financial Economics}}
  \bibinfo{volume}{7}, \bibinfo{number}{2} (\bibinfo{year}{2013}),
  \bibinfo{pages}{131--166}.
\newblock


\bibitem[\protect\citeauthoryear{Chen, Zhang, Liu, and Hu}{Chen
  et~al\mbox{.}}{2022}]%
        {chen2022individual}
\bibfield{author}{\bibinfo{person}{Yang Chen}, \bibinfo{person}{Libo Zhang},
  \bibinfo{person}{Jiamou Liu}, {and} \bibinfo{person}{Shuyue Hu}.}
  \bibinfo{year}{2022}\natexlab{}.
\newblock \showarticletitle{Individual-level inverse reinforcement learning for
  mean field games}. In \bibinfo{booktitle}{\emph{Proceedings of the 21st
  International Conference on Autonomous Agents and Multi-agent Systems}}.
\newblock


\bibitem[\protect\citeauthoryear{Cui and Koeppl}{Cui and Koeppl}{2021}]%
        {cui2021approximately}
\bibfield{author}{\bibinfo{person}{Kai Cui} {and} \bibinfo{person}{Heinz
  Koeppl}.} \bibinfo{year}{2021}\natexlab{}.
\newblock \showarticletitle{Approximately Solving Mean Field Games via
  Entropy-Regularized Deep Reinforcement Learning}. In
  \bibinfo{booktitle}{\emph{International Conference on Artificial Intelligence
  and Statistics}}. PMLR, \bibinfo{pages}{1909--1917}.
\newblock


\bibitem[\protect\citeauthoryear{Daskalakis, Goldberg, and
  Papadimitriou}{Daskalakis et~al\mbox{.}}{2009}]%
        {daskalakis2009complexity}
\bibfield{author}{\bibinfo{person}{Constantinos Daskalakis},
  \bibinfo{person}{Paul~W Goldberg}, {and} \bibinfo{person}{Christos~H
  Papadimitriou}.} \bibinfo{year}{2009}\natexlab{}.
\newblock \showarticletitle{The complexity of computing a Nash equilibrium}.
\newblock \bibinfo{journal}{\emph{SIAM J. Comput.}} \bibinfo{volume}{39},
  \bibinfo{number}{1} (\bibinfo{year}{2009}), \bibinfo{pages}{195--259}.
\newblock


\bibitem[\protect\citeauthoryear{Devlin and Kudenko}{Devlin and
  Kudenko}{2011}]%
        {devlin2011theoretical}
\bibfield{author}{\bibinfo{person}{Sam Devlin} {and} \bibinfo{person}{Daniel
  Kudenko}.} \bibinfo{year}{2011}\natexlab{}.
\newblock \showarticletitle{Theoretical considerations of potential-based
  reward shaping for multi-agent systems}. In \bibinfo{booktitle}{\emph{The
  10th International Conference on Autonomous Agents and Multiagent Systems}}.
  ACM, \bibinfo{pages}{225--232}.
\newblock


\bibitem[\protect\citeauthoryear{Elie, P{\'e}rolat, Lauri{\`e}re, Geist, and
  Pietquin}{Elie et~al\mbox{.}}{2020}]%
        {elie2020convergence}
\bibfield{author}{\bibinfo{person}{Romuald Elie}, \bibinfo{person}{Julien
  P{\'e}rolat}, \bibinfo{person}{Mathieu Lauri{\`e}re},
  \bibinfo{person}{Matthieu Geist}, {and} \bibinfo{person}{Olivier Pietquin}.}
  \bibinfo{year}{2020}\natexlab{}.
\newblock \showarticletitle{On the Convergence of Model Free Learning in Mean
  Field Games.}. In \bibinfo{booktitle}{\emph{Thirty-Fourth AAAI Conference on
  Artificial Intelligence}}. \bibinfo{pages}{7143--7150}.
\newblock


\bibitem[\protect\citeauthoryear{Fu, Luo, and Levine}{Fu et~al\mbox{.}}{2018}]%
        {fu2018learning}
\bibfield{author}{\bibinfo{person}{Justin Fu}, \bibinfo{person}{Katie Luo},
  {and} \bibinfo{person}{Sergey Levine}.} \bibinfo{year}{2018}\natexlab{}.
\newblock \showarticletitle{Learning Robust Rewards with Adverserial Inverse
  Reinforcement Learning}. In \bibinfo{booktitle}{\emph{International
  Conference on Learning Representations}}.
\newblock


\bibitem[\protect\citeauthoryear{Fu, Tacchetti, Perolat, and Bachrach}{Fu
  et~al\mbox{.}}{2021}]%
        {fu2021evaluating}
\bibfield{author}{\bibinfo{person}{Justin Fu}, \bibinfo{person}{Andrea
  Tacchetti}, \bibinfo{person}{Julien Perolat}, {and} \bibinfo{person}{Yoram
  Bachrach}.} \bibinfo{year}{2021}\natexlab{}.
\newblock \showarticletitle{Evaluating Strategic Structures in Multi-Agent
  Inverse Reinforcement Learning}.
\newblock \bibinfo{journal}{\emph{Journal of Artificial Intelligence Research}}
   \bibinfo{volume}{71} (\bibinfo{year}{2021}), \bibinfo{pages}{925--951}.
\newblock


\bibitem[\protect\citeauthoryear{Fudenberg and Tirole}{Fudenberg and
  Tirole}{1991}]%
        {fudenberg1991game}
\bibfield{author}{\bibinfo{person}{Drew Fudenberg} {and} \bibinfo{person}{Jean
  Tirole}.} \bibinfo{year}{1991}\natexlab{}.
\newblock \bibinfo{booktitle}{\emph{Game theory}}.
\newblock \bibinfo{publisher}{MIT press}.
\newblock


\bibitem[\protect\citeauthoryear{Gabaix}{Gabaix}{2014}]%
        {gabaix2014sparsity}
\bibfield{author}{\bibinfo{person}{Xavier Gabaix}.}
  \bibinfo{year}{2014}\natexlab{}.
\newblock \showarticletitle{A sparsity-based model of bounded rationality}.
\newblock \bibinfo{journal}{\emph{The Quarterly Journal of Economics}}
  \bibinfo{volume}{129}, \bibinfo{number}{4} (\bibinfo{year}{2014}),
  \bibinfo{pages}{1661--1710}.
\newblock


\bibitem[\protect\citeauthoryear{Garcia, Lucena, Zambonelli, Omicini, and
  Castro}{Garcia et~al\mbox{.}}{2002}]%
        {garcia2002software}
\bibfield{author}{\bibinfo{person}{Alessandro Garcia}, \bibinfo{person}{Carlos
  Lucena}, \bibinfo{person}{Franco Zambonelli}, \bibinfo{person}{Andrea
  Omicini}, {and} \bibinfo{person}{Jaelson Castro}.}
  \bibinfo{year}{2002}\natexlab{}.
\newblock \showarticletitle{Software Engineering for Large-Scale Multi-Agent
  Systems Research Issues and Practical Applications}. In
  \bibinfo{booktitle}{\emph{Conference proceedings SELMAS}}. Springer,
  \bibinfo{pages}{154}.
\newblock


\bibitem[\protect\citeauthoryear{Gomes, Mohr, and Souza}{Gomes
  et~al\mbox{.}}{2010}]%
        {gomes2010discrete}
\bibfield{author}{\bibinfo{person}{Diogo~A Gomes}, \bibinfo{person}{Joana
  Mohr}, {and} \bibinfo{person}{Rafael~Rigao Souza}.}
  \bibinfo{year}{2010}\natexlab{}.
\newblock \showarticletitle{Discrete time, finite state space mean field
  games}.
\newblock \bibinfo{journal}{\emph{Journal de Math{\'e}matiques Pures et
  Appliqu{\'e}es}} \bibinfo{volume}{93}, \bibinfo{number}{3}
  (\bibinfo{year}{2010}), \bibinfo{pages}{308--328}.
\newblock


\bibitem[\protect\citeauthoryear{Goodfellow, Pouget-Abadie, Mirza, Xu,
  Warde-Farley, Ozair, Courville, and Bengio}{Goodfellow et~al\mbox{.}}{2014}]%
        {goodfellow2014generative}
\bibfield{author}{\bibinfo{person}{Ian Goodfellow}, \bibinfo{person}{Jean
  Pouget-Abadie}, \bibinfo{person}{Mehdi Mirza}, \bibinfo{person}{Bing Xu},
  \bibinfo{person}{David Warde-Farley}, \bibinfo{person}{Sherjil Ozair},
  \bibinfo{person}{Aaron Courville}, {and} \bibinfo{person}{Yoshua Bengio}.}
  \bibinfo{year}{2014}\natexlab{}.
\newblock \showarticletitle{Generative adversarial nets}.
\newblock \bibinfo{journal}{\emph{Advances in neural information processing
  systems}}  \bibinfo{volume}{27} (\bibinfo{year}{2014}).
\newblock


\bibitem[\protect\citeauthoryear{Guo, Hu, Xu, and Zhang}{Guo
  et~al\mbox{.}}{2019}]%
        {guo2019learning}
\bibfield{author}{\bibinfo{person}{Xin Guo}, \bibinfo{person}{Anran Hu},
  \bibinfo{person}{Renyuan Xu}, {and} \bibinfo{person}{Junzi Zhang}.}
  \bibinfo{year}{2019}\natexlab{}.
\newblock \showarticletitle{Learning mean-field games}. In
  \bibinfo{booktitle}{\emph{Advances in Neural Information Processing
  Systems}}. \bibinfo{pages}{4967--4977}.
\newblock


\bibitem[\protect\citeauthoryear{Guo, Xu, and Zariphopoulou}{Guo
  et~al\mbox{.}}{2020}]%
        {guo2020entropy}
\bibfield{author}{\bibinfo{person}{Xin Guo}, \bibinfo{person}{Renyuan Xu},
  {and} \bibinfo{person}{Thaleia Zariphopoulou}.}
  \bibinfo{year}{2020}\natexlab{}.
\newblock \showarticletitle{Entropy regularization for mean field games with
  learning}.
\newblock \bibinfo{journal}{\emph{arXiv preprint arXiv:2010.00145}}
  (\bibinfo{year}{2020}).
\newblock


\bibitem[\protect\citeauthoryear{Haarnoja, Tang, Abbeel, and Levine}{Haarnoja
  et~al\mbox{.}}{2017}]%
        {haarnoja2017reinforcement}
\bibfield{author}{\bibinfo{person}{Tuomas Haarnoja}, \bibinfo{person}{Haoran
  Tang}, \bibinfo{person}{Pieter Abbeel}, {and} \bibinfo{person}{Sergey
  Levine}.} \bibinfo{year}{2017}\natexlab{}.
\newblock \showarticletitle{Reinforcement learning with deep energy-based
  policies}. In \bibinfo{booktitle}{\emph{International Conference on Machine
  Learning}}. PMLR, \bibinfo{pages}{1352--1361}.
\newblock


\bibitem[\protect\citeauthoryear{Harstad and Selten}{Harstad and
  Selten}{2013}]%
        {harstad2013bounded}
\bibfield{author}{\bibinfo{person}{Ronald~M Harstad} {and}
  \bibinfo{person}{Reinhard Selten}.} \bibinfo{year}{2013}\natexlab{}.
\newblock \showarticletitle{Bounded-rationality models: tasks to become
  intellectually competitive}.
\newblock \bibinfo{journal}{\emph{Journal of Economic Literature}}
  \bibinfo{volume}{51}, \bibinfo{number}{2} (\bibinfo{year}{2013}),
  \bibinfo{pages}{496--511}.
\newblock


\bibitem[\protect\citeauthoryear{Huang and Ma}{Huang and Ma}{2016}]%
        {huang2016mean}
\bibfield{author}{\bibinfo{person}{Minyi Huang} {and} \bibinfo{person}{Yan
  Ma}.} \bibinfo{year}{2016}\natexlab{}.
\newblock \showarticletitle{Mean field stochastic games: Monotone costs and
  threshold policies}. In \bibinfo{booktitle}{\emph{2016 IEEE 55th Conference
  on Decision and Control (CDC)}}. IEEE, \bibinfo{pages}{7105--7110}.
\newblock


\bibitem[\protect\citeauthoryear{Huang and Ma}{Huang and Ma}{2017}]%
        {huang2017mean}
\bibfield{author}{\bibinfo{person}{Minyi Huang} {and} \bibinfo{person}{Yan
  Ma}.} \bibinfo{year}{2017}\natexlab{}.
\newblock \showarticletitle{Mean field stochastic games with binary actions:
  Stationary threshold policies}. In \bibinfo{booktitle}{\emph{2017 IEEE 56th
  Annual Conference on Decision and Control (CDC)}}. IEEE,
  \bibinfo{pages}{27--32}.
\newblock


\bibitem[\protect\citeauthoryear{Huang, Malham{\'e}, Caines,
  et~al\mbox{.}}{Huang et~al\mbox{.}}{2006}]%
        {huang2006large}
\bibfield{author}{\bibinfo{person}{Minyi Huang}, \bibinfo{person}{Roland~P
  Malham{\'e}}, \bibinfo{person}{Peter~E Caines}, {et~al\mbox{.}}}
  \bibinfo{year}{2006}\natexlab{}.
\newblock \showarticletitle{Large population stochastic dynamic games:
  closed-loop McKean-Vlasov systems and the Nash certainty equivalence
  principle}.
\newblock \bibinfo{journal}{\emph{Communications in Information \& Systems}}
  \bibinfo{volume}{6}, \bibinfo{number}{3} (\bibinfo{year}{2006}),
  \bibinfo{pages}{221--252}.
\newblock


\bibitem[\protect\citeauthoryear{Lasry and Lions}{Lasry and Lions}{2007}]%
        {lasry2007mean}
\bibfield{author}{\bibinfo{person}{Jean-Michel Lasry} {and}
  \bibinfo{person}{Pierre-Louis Lions}.} \bibinfo{year}{2007}\natexlab{}.
\newblock \showarticletitle{Mean field games}.
\newblock \bibinfo{journal}{\emph{Japanese Journal of Mathematics}}
  \bibinfo{volume}{2}, \bibinfo{number}{1} (\bibinfo{year}{2007}),
  \bibinfo{pages}{229--260}.
\newblock


\bibitem[\protect\citeauthoryear{Lee, Rengarajan, Kalathil, and Shakkottai}{Lee
  et~al\mbox{.}}{2021}]%
        {lee2021reinforcement}
\bibfield{author}{\bibinfo{person}{Kiyeob Lee}, \bibinfo{person}{Desik
  Rengarajan}, \bibinfo{person}{Dileep Kalathil}, {and}
  \bibinfo{person}{Srinivas Shakkottai}.} \bibinfo{year}{2021}\natexlab{}.
\newblock \showarticletitle{Reinforcement learning for mean field games with
  strategic complementarities}. In \bibinfo{booktitle}{\emph{International
  Conference on Artificial Intelligence and Statistics}}. PMLR,
  \bibinfo{pages}{2458--2466}.
\newblock


\bibitem[\protect\citeauthoryear{Levy and Solan}{Levy and Solan}{2020}]%
        {levy2020stochastic}
\bibfield{author}{\bibinfo{person}{Yehuda~John Levy} {and}
  \bibinfo{person}{Eilon Solan}.} \bibinfo{year}{2020}\natexlab{}.
\newblock \showarticletitle{Stochastic games}.
\newblock \bibinfo{journal}{\emph{Complex Social and Behavioral Systems: Game
  Theory and Agent-Based Models}} (\bibinfo{year}{2020}),
  \bibinfo{pages}{229--250}.
\newblock


\bibitem[\protect\citeauthoryear{Morris-Martin, De~Vos, and
  Padget}{Morris-Martin et~al\mbox{.}}{2019}]%
        {morris2019norm}
\bibfield{author}{\bibinfo{person}{Andreasa Morris-Martin},
  \bibinfo{person}{Marina De~Vos}, {and} \bibinfo{person}{Julian Padget}.}
  \bibinfo{year}{2019}\natexlab{}.
\newblock \showarticletitle{Norm emergence in multiagent systems: a viewpoint
  paper}.
\newblock \bibinfo{journal}{\emph{Autonomous Agents and Multi-Agent Systems}}
  \bibinfo{volume}{33}, \bibinfo{number}{6} (\bibinfo{year}{2019}),
  \bibinfo{pages}{706--749}.
\newblock


\bibitem[\protect\citeauthoryear{Munz, Papachristodoulou, and Allgower}{Munz
  et~al\mbox{.}}{2008}]%
        {munz2008delay}
\bibfield{author}{\bibinfo{person}{Ulrich Munz}, \bibinfo{person}{Antonis
  Papachristodoulou}, {and} \bibinfo{person}{Frank Allgower}.}
  \bibinfo{year}{2008}\natexlab{}.
\newblock \showarticletitle{Delay-dependent rendezvous and flocking of large
  scale multi-agent systems with communication delays}. In
  \bibinfo{booktitle}{\emph{2008 47th IEEE Conference on Decision and
  Control}}. IEEE, \bibinfo{pages}{2038--2043}.
\newblock


\bibitem[\protect\citeauthoryear{Natarajan, Kunapuli, Judah, Tadepalli,
  Kersting, and Shavlik}{Natarajan et~al\mbox{.}}{2010}]%
        {natarajan2010multi}
\bibfield{author}{\bibinfo{person}{Sriraam Natarajan}, \bibinfo{person}{Gautam
  Kunapuli}, \bibinfo{person}{Kshitij Judah}, \bibinfo{person}{Prasad
  Tadepalli}, \bibinfo{person}{Kristian Kersting}, {and} \bibinfo{person}{Jude
  Shavlik}.} \bibinfo{year}{2010}\natexlab{}.
\newblock \showarticletitle{Multi-agent inverse reinforcement learning}. In
  \bibinfo{booktitle}{\emph{2010 ninth international conference on machine
  learning and applications}}. IEEE, \bibinfo{pages}{395--400}.
\newblock


\bibitem[\protect\citeauthoryear{Ng, Harada, and Russell}{Ng
  et~al\mbox{.}}{1999}]%
        {ng1999policy}
\bibfield{author}{\bibinfo{person}{Andrew~Y Ng}, \bibinfo{person}{Daishi
  Harada}, {and} \bibinfo{person}{Stuart Russell}.}
  \bibinfo{year}{1999}\natexlab{}.
\newblock \showarticletitle{Policy invariance under reward transformations:
  Theory and application to reward shaping}. In
  \bibinfo{booktitle}{\emph{ICML}}, Vol.~\bibinfo{volume}{99}.
  \bibinfo{pages}{278--287}.
\newblock


\bibitem[\protect\citeauthoryear{Ng and Russell}{Ng and Russell}{2000}]%
        {ng2000algorithms}
\bibfield{author}{\bibinfo{person}{Andrew~Y Ng} {and} \bibinfo{person}{Stuart~J
  Russell}.} \bibinfo{year}{2000}\natexlab{}.
\newblock \showarticletitle{Algorithms for Inverse Reinforcement Learning}. In
  \bibinfo{booktitle}{\emph{Proceedings of the Seventeenth International
  Conference on Machine Learning}}. \bibinfo{pages}{663--670}.
\newblock


\bibitem[\protect\citeauthoryear{Ortega and Braun}{Ortega and Braun}{2011}]%
        {ortega2011information}
\bibfield{author}{\bibinfo{person}{Daniel~Alexander Ortega} {and}
  \bibinfo{person}{Pedro~Alejandro Braun}.} \bibinfo{year}{2011}\natexlab{}.
\newblock \showarticletitle{Information, utility and bounded rationality}. In
  \bibinfo{booktitle}{\emph{International Conference on Artificial General
  Intelligence}}. Springer, \bibinfo{pages}{269--274}.
\newblock


\bibitem[\protect\citeauthoryear{Perrin, Lauri{\`e}re, P{\'e}rolat, {\'E}lie,
  Geist, and Pietquin}{Perrin et~al\mbox{.}}{2022}]%
        {perrin2022generalization}
\bibfield{author}{\bibinfo{person}{Sarah Perrin}, \bibinfo{person}{Mathieu
  Lauri{\`e}re}, \bibinfo{person}{Julien P{\'e}rolat}, \bibinfo{person}{Romuald
  {\'E}lie}, \bibinfo{person}{Matthieu Geist}, {and} \bibinfo{person}{Olivier
  Pietquin}.} \bibinfo{year}{2022}\natexlab{}.
\newblock \showarticletitle{Generalization in mean field games by learning
  master policies}. In \bibinfo{booktitle}{\emph{Proceedings of the AAAI
  Conference on Artificial Intelligence}}, Vol.~\bibinfo{volume}{36}.
  \bibinfo{pages}{9413--9421}.
\newblock


\bibitem[\protect\citeauthoryear{Ratliff, Bagnell, and Zinkevich}{Ratliff
  et~al\mbox{.}}{2006}]%
        {ratliff2006maximum}
\bibfield{author}{\bibinfo{person}{Nathan~D Ratliff}, \bibinfo{person}{J~Andrew
  Bagnell}, {and} \bibinfo{person}{Martin~A Zinkevich}.}
  \bibinfo{year}{2006}\natexlab{}.
\newblock \showarticletitle{Maximum margin planning}. In
  \bibinfo{booktitle}{\emph{Proceedings of the 23rd International Conference on
  Machine Learning}}. \bibinfo{pages}{729--736}.
\newblock


\bibitem[\protect\citeauthoryear{Saldi, Basar, and Raginsky}{Saldi
  et~al\mbox{.}}{2018}]%
        {saldi2018markov}
\bibfield{author}{\bibinfo{person}{Naci Saldi}, \bibinfo{person}{Tamer Basar},
  {and} \bibinfo{person}{Maxim Raginsky}.} \bibinfo{year}{2018}\natexlab{}.
\newblock \showarticletitle{Markov--Nash Equilibria in Mean-Field Games with
  Discounted Cost}.
\newblock \bibinfo{journal}{\emph{SIAM Journal on Control and Optimization}}
  \bibinfo{volume}{56}, \bibinfo{number}{6} (\bibinfo{year}{2018}),
  \bibinfo{pages}{4256--4287}.
\newblock


\bibitem[\protect\citeauthoryear{Shapley}{Shapley}{1964}]%
        {shapley1964some}
\bibfield{author}{\bibinfo{person}{Lloyd Shapley}.}
  \bibinfo{year}{1964}\natexlab{}.
\newblock \showarticletitle{Some topics in two-person games}.
\newblock \bibinfo{journal}{\emph{Advances in game theory}}
  \bibinfo{volume}{52} (\bibinfo{year}{1964}), \bibinfo{pages}{1--29}.
\newblock


\bibitem[\protect\citeauthoryear{Song, Ren, Sadigh, and Ermon}{Song
  et~al\mbox{.}}{2018}]%
        {song2018multi}
\bibfield{author}{\bibinfo{person}{Jiaming Song}, \bibinfo{person}{Hongyu Ren},
  \bibinfo{person}{Dorsa Sadigh}, {and} \bibinfo{person}{Stefano Ermon}.}
  \bibinfo{year}{2018}\natexlab{}.
\newblock \showarticletitle{Multi-agent generative adversarial imitation
  learning}. In \bibinfo{booktitle}{\emph{Advances in Neural Information
  Processing Systems}}. \bibinfo{pages}{7461--7472}.
\newblock


\bibitem[\protect\citeauthoryear{Subramanian and Mahajan}{Subramanian and
  Mahajan}{2019}]%
        {subramanian2019reinforcement}
\bibfield{author}{\bibinfo{person}{Jayakumar Subramanian} {and}
  \bibinfo{person}{Aditya Mahajan}.} \bibinfo{year}{2019}\natexlab{}.
\newblock \showarticletitle{Reinforcement learning in stationary mean-field
  games}. In \bibinfo{booktitle}{\emph{Proceedings of the 18th International
  Conference on Autonomous Agents and Multi-agent Systems}}.
  \bibinfo{pages}{251--259}.
\newblock


\bibitem[\protect\citeauthoryear{Subramanian, Poupart, Taylor, and
  Hegde}{Subramanian et~al\mbox{.}}{2020}]%
        {ganapathi2020multi}
\bibfield{author}{\bibinfo{person}{Sriram Subramanian}, \bibinfo{person}{Pascal
  Poupart}, \bibinfo{person}{Matthew~E Taylor}, {and} \bibinfo{person}{Nidhi
  Hegde}.} \bibinfo{year}{2020}\natexlab{}.
\newblock \showarticletitle{Multi Type Mean Field Reinforcement Learning}. In
  \bibinfo{booktitle}{\emph{Proceedings of the 19th International Conference on
  Autonomous Agents and Multi-agent Systems}}. \bibinfo{pages}{411--419}.
\newblock


\bibitem[\protect\citeauthoryear{Weintraub, Benkard, and Van~Roy}{Weintraub
  et~al\mbox{.}}{2010}]%
        {weintraub2010computational}
\bibfield{author}{\bibinfo{person}{Gabriel~Y Weintraub},
  \bibinfo{person}{C~Lanier Benkard}, {and} \bibinfo{person}{Benjamin
  Van~Roy}.} \bibinfo{year}{2010}\natexlab{}.
\newblock \showarticletitle{Computational methods for oblivious equilibrium}.
\newblock \bibinfo{journal}{\emph{Operations research}} \bibinfo{volume}{58},
  \bibinfo{number}{4-part-2} (\bibinfo{year}{2010}),
  \bibinfo{pages}{1247--1265}.
\newblock


\bibitem[\protect\citeauthoryear{Xie, Yang, Wang, and Minca}{Xie
  et~al\mbox{.}}{2021}]%
        {xie2021learning}
\bibfield{author}{\bibinfo{person}{Qiaomin Xie}, \bibinfo{person}{Zhuoran
  Yang}, \bibinfo{person}{Zhaoran Wang}, {and} \bibinfo{person}{Andreea
  Minca}.} \bibinfo{year}{2021}\natexlab{}.
\newblock \showarticletitle{Learning while playing in mean-field games:
  Convergence and optimality}. In \bibinfo{booktitle}{\emph{International
  Conference on Machine Learning}}. PMLR, \bibinfo{pages}{11436--11447}.
\newblock


\bibitem[\protect\citeauthoryear{Yang, Ye, Trivedi, Xu, and Zha}{Yang
  et~al\mbox{.}}{2018b}]%
        {yang2018learning}
\bibfield{author}{\bibinfo{person}{Jiachen Yang}, \bibinfo{person}{Xiaojing
  Ye}, \bibinfo{person}{Rakshit Trivedi}, \bibinfo{person}{Huan Xu}, {and}
  \bibinfo{person}{Hongyuan Zha}.} \bibinfo{year}{2018}\natexlab{b}.
\newblock \showarticletitle{Learning Deep Mean Field Games for Modeling Large
  Population Behavior}. In \bibinfo{booktitle}{\emph{International Conference
  on Learning Representations}}.
\newblock


\bibitem[\protect\citeauthoryear{Yang, Luo, Li, Zhou, Zhang, and Wang}{Yang
  et~al\mbox{.}}{2018a}]%
        {yang2018mean}
\bibfield{author}{\bibinfo{person}{Y Yang}, \bibinfo{person}{R Luo},
  \bibinfo{person}{M Li}, \bibinfo{person}{M Zhou}, \bibinfo{person}{W Zhang},
  {and} \bibinfo{person}{J Wang}.} \bibinfo{year}{2018}\natexlab{a}.
\newblock \showarticletitle{Mean Field Multi-Agent Reinforcement Learning}. In
  \bibinfo{booktitle}{\emph{35th International Conference on Machine
  Learning}}, Vol.~\bibinfo{volume}{80}. PMLR, \bibinfo{pages}{5571--5580}.
\newblock


\bibitem[\protect\citeauthoryear{Yu, Song, and Ermon}{Yu
  et~al\mbox{.}}{2019a}]%
        {yu2019multi}
\bibfield{author}{\bibinfo{person}{Lantao Yu}, \bibinfo{person}{Jiaming Song},
  {and} \bibinfo{person}{Stefano Ermon}.} \bibinfo{year}{2019}\natexlab{a}.
\newblock \showarticletitle{Multi-Agent Adversarial Inverse Reinforcement
  Learning}. In \bibinfo{booktitle}{\emph{International Conference on Machine
  Learning}}. \bibinfo{pages}{7194--7201}.
\newblock


\bibitem[\protect\citeauthoryear{Yu, Yu, Finn, and Ermon}{Yu
  et~al\mbox{.}}{2019b}]%
        {yu2019meta}
\bibfield{author}{\bibinfo{person}{Lantao Yu}, \bibinfo{person}{Tianhe Yu},
  \bibinfo{person}{Chelsea Finn}, {and} \bibinfo{person}{Stefano Ermon}.}
  \bibinfo{year}{2019}\natexlab{b}.
\newblock \showarticletitle{Meta-inverse reinforcement learning with
  probabilistic context variables}.
\newblock \bibinfo{journal}{\emph{Advances in Neural Information Processing
  Systems}}  \bibinfo{volume}{32} (\bibinfo{year}{2019}).
\newblock


\bibitem[\protect\citeauthoryear{Zhou, Chen, Wen, Yang, Su, Zhang, Zhang, and
  Wang}{Zhou et~al\mbox{.}}{2019}]%
        {zhou2019factorized}
\bibfield{author}{\bibinfo{person}{Ming Zhou}, \bibinfo{person}{Yong Chen},
  \bibinfo{person}{Ying Wen}, \bibinfo{person}{Yaodong Yang},
  \bibinfo{person}{Yufeng Su}, \bibinfo{person}{Weinan Zhang},
  \bibinfo{person}{Dell Zhang}, {and} \bibinfo{person}{Jun Wang}.}
  \bibinfo{year}{2019}\natexlab{}.
\newblock \showarticletitle{Factorized q-learning for large-scale multi-agent
  systems}. In \bibinfo{booktitle}{\emph{Proceedings of the First International
  Conference on Distributed Artificial Intelligence}}. \bibinfo{pages}{1--7}.
\newblock


\bibitem[\protect\citeauthoryear{Ziebart, Bagnell, and Dey}{Ziebart
  et~al\mbox{.}}{2010}]%
        {ziebart2010modeling}
\bibfield{author}{\bibinfo{person}{Brian~D Ziebart}, \bibinfo{person}{J~Andrew
  Bagnell}, {and} \bibinfo{person}{Anind~K Dey}.}
  \bibinfo{year}{2010}\natexlab{}.
\newblock \showarticletitle{Modeling interaction via the principle of maximum
  causal entropy}. In \bibinfo{booktitle}{\emph{Proceedings of the 27th
  International Conference on Machine Learning}}. \bibinfo{pages}{1255--1262}.
\newblock


\bibitem[\protect\citeauthoryear{Ziebart, Maas, Bagnell, and Dey}{Ziebart
  et~al\mbox{.}}{2008}]%
        {ziebart2008maximum}
\bibfield{author}{\bibinfo{person}{Brian~D Ziebart}, \bibinfo{person}{Andrew
  Maas}, \bibinfo{person}{J~Andrew Bagnell}, {and} \bibinfo{person}{Anind~K
  Dey}.} \bibinfo{year}{2008}\natexlab{}.
\newblock \showarticletitle{Maximum entropy inverse reinforcement learning}. In
  \bibinfo{booktitle}{\emph{Proceedings of the 23rd AAAI Conference on
  Artificial Intelligence}}. \bibinfo{pages}{1433--1438}.
\newblock


\end{thebibliography}

%%%%%%%%%%%%%%%%%%%%%%%%%%%%%%%%%%%%%%%%%%%%%%%%%%%%%%%%%%%%%%%%%%%%%%%%
\newpage
\onecolumn

\setcounter{section}{0}
\setcounter{equation}{0}
\renewcommand{\theequation}{\thesection\arabic{equation}}
\newtheorem{lem}{Lemma}[section]
\newtheorem{cor}{Corollary}[section]
\newtheorem{prop}{Proposition}[section]
\newcommand{\boltz}{B_{\beta}}

\appendix

\begin{center}
	{\LARGE\bf Appendices}
\end{center}

\section{Proof of Theorem~\ref{thm:trajectory_ERMFNE}}\label{proof:trajectory_ERMFNE}

Our proof of Theorem~\ref{thm:trajectory_ERMFNE} relies on the following lemma from \citep{cui2021approximately}, which shows that an energy-based model can characterise the policy in ERMFNE in terms of action values.

\begin{lem}[\citep{cui2021approximately}]\label{lem:ebm}
Let $(\Smc, \Amc, P, \mu_0, r)$ be an MFG with the entropy regularisation and $(\tilde{\muv}^\star, \tilde{\piv}^\star)$ be the ERMFNE. 
	Denote the {\em action value function} (i.e., cumulative future rewards of selecting an action in a state) of $(\muv, \piv)$ by
	$$Q^{\muv, \piv_{t+1:T-1}}(s_t,a_t,\mu_t) \triangleq r(s_t, a_t, \mu_t) + \\
	\mathbb{E}_{\piv_{t+1:T-1}} \left[ \sum_{\ell = t+1}^{T-1}  r(s_\ell, a_\ell, \mu_\ell) +  \Hmc(\pi_\ell (\cdot \vert s_\ell)) \right].$$ Then, the policy $\tilde{\piv}^\star$ is in the form of 
	$$\tilde{\pi}^\star_t (a_t \vert s_t) = \frac{ e^{ Q^{\tilde{\muv}^\star, \tilde{\piv}_{t+1:T-1}}(s_t,a_t,\tilde{\mu}^\star_t) }}{\sum_{a' \in \Amc} e^{Q^{\tilde{\muv}^\star, \tilde{\piv}_{t+1:T-1}}(s_t,a',\tilde{\mu}^\star_t)}}.$$
\end{lem}

\subsection{Proof of Theorem~\ref{thm:trajectory_ERMFNE}}
\begin{proof}
%Since all agents use the same policy, the state distribution of an individual agent must be consistent with the population's state distribution, i.e., the mean field. Hence, Eq.~\eqref{eq:prop:1} holds. 

%We then show that Eq.~\eqref{eq:prop:2} holds.
Let $(\tilde{\muv}^\star, \tilde{\piv}^\star)$ be the ERMFNE for a MFG $(\Smc, \Amc, p, \mu_0, r)$. For an arbitrary policy $\piv$, the probability of a trajectory $\tau = \{(s_t,a_t)\}_{t=0}^{T-1}$ induced by $(\muv, \piv)$ satisfies the following distribution:
\begin{equation}
	p_1(\tau) = \mu_0(s_0) \cdot \prod_{t=0}^{T-1} P(s_{t+1} \vert s_t, a_t, \mu_t) \cdot \prod_{t=0}^T \pi_t(a_t \vert s_t).
\end{equation}

Recall our desired energy-based trajectory distribution formula is 
$$p_2(\tau) \propto \mu_0(s_0) \cdot \prod_{t=0}^{T-1} P(s_{t+1} \vert s_t, a_t, \mu_t) \cdot e^{R(\tau)}.$$

Let $D_{\mathrm{KL}}$ denote the Kullback-Leibler (KL) divergence. We now show that the  ERMFNE $(\tilde{\muv}^\star, \tilde{\piv}^\star)$ is the optimal solution to the following optimisation problem: 

\begin{equation}\label{eq:thm2}
	\min_{\muv, \piv}D_{\mathrm{KL}}\left(p_1(\tau) \parallel p_2(\tau)\right) \text{\em~ s.t.~ } \muv = \Phi(\piv).
\end{equation}

The constraint enforces the condition of population consistency. Thus, fixing $\muv$ as $\tilde{\muv}^\star$, to show the ERMFNE $(\tilde{\muv}^\star, \tilde{\piv}^\star)$ is the optimal solution to Eq.~\eqref{eq:thm2} is equivalent to show that $\tilde{\piv}^\star$ is the optimal solution to the following optimisation problem:

\begin{equation}\label{eq:KL}
	\min_{\piv}D_{\mathrm{KL}}\left(p_1(\tau) \parallel p_2(\tau)\right) \text{\em~ where~ } \muv = \tilde{\muv}^\star.
\end{equation}

Our proof is in a manner of dynamic programming. We construct the basic case for step $T-1$, where Eq.~\eqref{eq:KL} holds according to the definition of the policy in ERMFNE.  Then, for each $t<T-1$, we construct the optimal policy for steps from $t$ to $T-1$ based on the optimal policy that is already constructed from $t+1$ to $T-1$. We show that the constructed optimal policy that minimises the above KL divergence is consistent with $\tilde{\piv}^\star$ in ERMFNE. We next elaborate on the procedure of dynamic programming.

For simplicity, we omit the partition function for $p_2$ as it is a constant. Substituting $p_1(\tau)$ and $p_2(\tau)$ in Eq.~\eqref{eq:KL} with their definitions and roll out the KL-divergence, we obtain that maximising the KL-divergence between $p_1(\tau)$ and $p_2(\tau)$ is equivalent to the following optimisation problem
	\begin{equation}\label{eq:thm3_1}
	\begin{aligned}
		\max_{\piv}\mathbb{E}_{\tau \sim p_1} \left[ \log \frac{p_2(\tau)}{p_1(\tau)} \right] = \mathbb{E}_{\tau \sim  p_1}\Bigg[ & \log \mu_0(s_0) + \sum_{t=0}^{T-1} \left(  r(s_t, a_t, \tilde{\mu}^\star_t) + \log P(s_{t+1} \vert s_t, a_t, \tilde{\mu}^\star_t) \right) - \\
		& \log \mu_0(s_0) - \sum_{t=0}^{T-1} \left( \log \pi_t(a_t \vert s_t) + \log P(s_{t+1} \vert s_t, a_t, \tilde{\mu}^\star_t) \right) \Bigg]\\
		= \mathbb{E}_{\tau \sim  p_1} \Bigg[ & \sum_{t=0}^{T-1}  r(s_t, a_t, \tilde{\mu}^\star_t) - \log \pi_t(a_t \vert s_t)   \Bigg]
	\end{aligned}
	\end{equation}
	We maximise the objective in Eq.~\eqref{eq:thm3_1} using a dynamic programming method. 
	
	Consider the terminal step $t=T$, since the reward at the terminal step is always $0$, to maximise the entropy, the policy $\pi_T$ chooses actions evenly, i.e, $\tilde{\pi}^\star_T(a \vert s) = 1/ |\Amc|$ for any $a \in \Amc$ and $s \in \Smc$.
	
	We then construct the base case that maximises $\pi_{T-1}$:
	\begin{equation}\label{eq:thm3_2}
	\begin{aligned}
		\max_{\pi_{T-1}} ~& \mathbb{E}_{(s_{T-1}, a_{T-1}) \sim p_1} \left[ r(s_{T-1}, a_{T-1}, \tilde{\mu}^\star_{T-1}) - \log \pi_{T-1}(a_{T-1} \vert s_{T-1}) \right]\\
		= & \mathbb{E}_{(s_{T-1}, a_{T-1}) \sim p_1} \Bigg[ -D_{\mathrm{KL}} \left( \pi_{T-1}(a_{T-1} \vert s_{T-1}) \bigg\| \frac{e^{r(s_{T-1}, a_{T-1}, \tilde{\mu}^\star_{T-1})}}{ e^{ V(s_{T-1}, \tilde{\mu}^\star_{T-1})} } \right) + V(s_{T-1}, \tilde{\mu}^\star_{T-1}) \Bigg],
	\end{aligned}
	\end{equation}
	where we define $$V(s_{T-1}, \tilde{\mu}^\star_{T-1}) \triangleq \log \sum_{a' \in \Amc} e^{r(s_{T-1}, a', \tilde{\mu}^\star_{T-1})}.$$ 
	
	The optimal $\pi_{T-1}$ for Eq.~\eqref{eq:thm3_2} is 
	\begin{equation}\label{eq:thm3_3}
		\tilde{\pi}_{T-1}^\star ( a_{T-1} \vert s_{T-1} ) = \frac{e^{r(s_{T-1}, a_{T-1}, \tilde{\mu}^\star_{T-1})}}{e^{V(s_{T-1}, \tilde{\mu}^\star_{T-1})}} = \frac{e^{ r(s_{T-1}, a_{T-1}, \tilde{\mu}^\star_{T-1}) }}{\sum_{a' \in \Amc} e^{ r(s_{T-1}, a', \tilde{\mu}^\star_{T-1})}},
	\end{equation}
	which coincides with the form given in Lemma~\ref{lem:ebm}.
	
	With the optimal policy given in Eq.~\eqref{eq:thm3_3}, the KL-divergence in Eq.~\eqref{eq:thm3_2} attains $0$ and Eq.~\eqref{eq:thm3_2} attains the minimum $V(s_{T-1}, \tilde{\mu}^\star_{T-1})$.
	
	Then recursively, we can show that for any step $t < T - 1$, $\pi_t$ is the maximiser of the following optimisation problem:
	\begin{equation}\label{eq:thm3_4}
	\begin{aligned}
		\max_{\pi_t} \mathbb{E}_{(s_t, a_t) \sim p_1} \left[ -D_{\mathrm{KL}} \left( \pi_t(a_t \vert s_t) \bigg\| \frac{e^{ Q^{\tilde{\muv}^\star, \tilde{\piv}^\star_{t+1:T-1}}(s_t,a_t,\tilde{\mu}_t^\star)}}{e^{ V^{\tilde{\muv}^\star, \tilde{\piv}^\star_{t+1:T-1}}(s_t, \tilde{\mu}_t^\star) }} \right) + V^{\tilde{\muv}^\star, \tilde{\piv}^\star_{t+1:T-1}}(s_t, \tilde{\mu}_t^\star) \right],
	\end{aligned}
	\end{equation}
	%Here, $Q^{\muv, \piv_{t+1:T-1}}_{\soft}(s, a, \mu^*_t)$ denotes the soft $Q$-function under MF flow $\muv^*$ and per-step policies $\pi_{t+1}, \ldots, \pi_{T-1}$ that we already derived by recursive solving Eq.~\eqref{eq:thm3_4}, and 
	where 
	$$V^{\tilde{\muv}^\star, \tilde{\piv}^\star_{t+1:T-1}}(s_t, \tilde{\mu}^\star_t) \triangleq \log \sum_{a' \in \Amc} e^{Q^{\tilde{\muv}^\star, \tilde{\piv}^\star_{t+1:T-1}}(s_t,a_t,\tilde{\mu}^\star_t)}.$$
	
	In fact, $V^{\muv, \piv_{t+1:T-1}}$ resembles the {\em soft value function} in {\em soft Q-learning} \citep{haarnoja2017reinforcement}.
	
	The optimal policy for Eq.~\eqref{eq:thm3_4} is given by $$\pi_t (a_t \vert s_t) = \frac{e^{Q^{\tilde{\muv}^\star, \tilde{\piv}^\star_{t+1:T-1}}(s_t, a_t, \tilde{\mu}^\star_t)}}{e^{V^{\tilde{\muv}^\star, \tilde{\piv}^\star_{t+1:T-1}}(s_t, \tilde{\mu}^\star_t)}} = \frac{e^{Q^{\tilde{\muv}^\star, \tilde{\piv}^\star_{t+1:T-1}}(s_t, a_t, \tilde{\mu}^\star_t) }}{\sum_{a \in \Amc} e^{Q^{\tilde{\muv}^\star, \tilde{\piv}^\star_{t+1:T-1}}(s_t, a, \tilde{\mu}^\star_t) }},$$
	which coincides with the form given in Lemma~\ref{lem:ebm}. 
\end{proof}

\section{Justifications for the Reward Shaping in MFGs}\label{app:shaping}
We show in the following theorem that for any potential function $g: \Smc \times \Pmc(\Smc) \to \Rbb$, the potential-based reward shaping can ensure the invariance of both ERMFNE and MFNE.

\begin{theorem}
	Let any $\Smc, \Amc$ be given. We say $F: \Smc \times \Amc \times \Pmc(\Smc) \times \Smc \times \Pmc(\Smc) \to \Rbb$ is a {\em potential-based reward shaping} for MFG if there exists a real-valued function $g: \Smc \times \Pmc(\Smc) \to \Rbb$ such that $F(s_t,a_t,\mu_t,s_{t+1},\mu_{t+1}) =  g(s_{t+1}, \mu_{t+1}) - g(s_t, \mu_t)$. Then, $F$ is sufficient and necessary to guarantee the invariance of the set of MFNE and ERMFNE in the sense that:
		\begin{itemize}
			\item {\bf Sufficiency:} Every MFNE or ERMFNE in the MFG $\Mmc' = (\Smc, \Amc, P, \mu_0, r + F)$ is also a MFNE or ERMFNE in $\Mmc = (\Smc, \Amc, P, \mu_0, r)$;
			\item {\bf Necessity:} If $F$ is not a potential-based reward shaping, then there exist an initial mean field $\mu_0$, transition function $P$, horizon $T$, and reward function $r$ such that no MFNE or ERMFNE in $\Mmc'$ is an equilibrium in $\Mmc$.
		\end{itemize}
\end{theorem}

\begin{proof}
	We first prove the sufficiency. %For a MF flow-policy pair $(\muv, \piv)$, we denote $Q(s_t, a_t)$ and $Q_\soft(s_t, a_t)$ the Q value functions and soft Q value functions of the MDP induced by $\muv$, respectively.
	First, we show the set of MFNE remains invariant under the potential-based reward shaping $F$. Let $\muv$ be an arbitrary MF flow. The optimal action function for the MFG induced by $\muv$, denoted by $Q^\star$, fulfil the Bellman equation:
		
	$$Q^\star(s_t, a_t) = r(s_t, a_t, \mu_t) + \mathbb{E}_{s_{t+1} \sim p} \left[ \max_{a_{t+1} \in \Amc} Q^\star(s_{t+1}, a_{t+1})  \right].$$
	
	Applying some simple algebraic manipulation gives:
	
	\begin{equation*}
	\begin{aligned}
		&~~~~~Q^\star(s_t, a_t) - g(s_t, \mu_t)\\
		 &= r(s_t, a_t, \mu_t) - g(s_t, \mu_t) +  g(s_{t+1}, \mu_{t+1}) +  \mathbb{E}_{s_{t+1} \sim P}\left[ \max_{a_{t+1} \in \Amc} ( Q^\star(s_{t+1}, a_{t+1}) - g(s_{t+1}, \mu_{t+1}) ) \right]\\ 
		 &=r(s_t, a_t, \mu_t) + F(s_t,a_t,\mu_t,s_{t+1},\mu_{t+1})  +  \mathbb{E}_{s_{t+1} \sim P}\left[ \max_{a_{t+1} \in \Amc} ( Q^\star(s_{t+1}, a_{t+1}) - g(s_{t+1}, \mu_{t+1}) ) \right].\\
	\end{aligned}
	\end{equation*}
	
	From here, we know that the above equation is exactly the Bellman equation induced by $\muv$ with the reward function $r + F$, and $Q^\star(s_t, a_t) - g(s_t, \mu_t)$ is the unique set of optimal Q values. Since $\argmax_{a\in \Amc} Q^\star(s_t, a_t) - g(s_t, \mu_t) = \argmax_{a\in \Amc} Q^\star(s_t, a_t)$, we have that for any fix a MF flow, any optimal policy of $\Mmc'$ is also an optimal policy of $\Mmc$. Hence, we finish the proof of the sufficiency of MFNE.
	
	Next, we will show the sufficiency of ERMFNE. We write $\tilde{Q}^\star$ for the optimal action-value functions for the MFG with the entropy regularisation such that
	\begin{equation*}
	\begin{aligned}
		\tilde{Q}^\star(s_t,a_t) = r(s_t, a_t, \mu_t) +  \mathbb{E}_{s_{t+1} \sim p} \left[ \sum_{a_{t+1} \in \Amc} \frac{e^{ \tilde{Q}^\star(s_{t+1}, a_{t+1})}}{\sum_{a' \in \Amc}e^{\tilde{Q}^\star(s_t,a')}}  \tilde{Q}^\star(s_{t+1}, a_{t+1}) \right].
	\end{aligned}
	\end{equation*}
	
	Using the same manner as we show the invariance of MFNE, we have:
	\begin{equation*}
		\begin{aligned}
			&~~~~~\tilde{Q}^\star(s_t,a_t) - g(s_t, \mu_t)\\
			&= r(s_t, a_t, \mu_t) - g(s_t, \mu_t) +  g(s_{t+1}, \mu_{t+1})\\
			&~~~~~ +  \Ebb_{s_{t+1} \sim P}\left[ \frac{e^{\tilde{Q}^\star(s_{t+1}, a_{t+1}) - g(s_{t+1}, \mu_{t+1}) }}{\sum_{a' \in \Amc}e^{\tilde{Q}^\star(s_t,a') - g(s_{t+1}, \mu_{t+1})}}  (\tilde{Q}^\star(s_{t+1}, a_{t+1}) - g(s_{t+1}, \mu_{t+1})) \right]\\
			&= r(s_t, a_t, \mu_t) - g(s_t, \mu_t) +  g(s_{t+1}, \mu_{t+1})\\
			&~~~~~ + \Ebb_{s_{t+1} \sim P}\left[ \frac{e^{ \tilde{Q}^\star(s_{t+1}, a_{t+1}) }}{\sum_{a' \in \Amc}e^{\tilde{Q}^\star(s_t,a')}}  (\tilde{Q}^\star(s_{t+1}, a_{t+1}) - g(s_{t+1}, \mu_{t+1})) \right].\\
		\end{aligned}
	\end{equation*}
	
	Hence, we know that $\tilde{Q}^\star(s_t,a_t) - g(s_t, \mu_t)$ is the set of optimal action values induced by $\muv$ under the reward function $r + F$. Thus, any optimal policy of $\Mmc'$ is also an optimal policy of $\Mmc$. Hence, we finish the proof of sufficiency for MFNE.
	
	We next show the necessity by constructing a counter-example where a non-potential-based reward shaping can change the set of MFNE and ERMFNE. Consider the {\em Left-Right} problem \citep{cui2021approximately}, which is also used as a task in experiments. At each step, each agent is at a position (state) of either ``left'', ``right'' or ``center'', and can choose to move either ``left'' or ``right'', receives a reward according the current population density (mean field) at each position, and with probability one (dynamics) reaches ``left'' or ``right''. Once an agent leaves the ``centre'', she can never head back and can only be in either left or right thereafter. Formally, we configure the MFG as follows: $\Smc = \{C, L, R\}$, $\Amc = \Smc \setminus \{C\}$, initial mean field $\mu_0(C) = 1$ (i.e., all agents are at ``center'' initially) and the reward function $$r(s,a,\mu_t) = -\mathds{1}_{\{s = L\}}\cdot \mu_t(L) -\mathds{1}_{\{s = R\}}\cdot \mu_t(R).$$ This reward setting means that each agent will incur a negative reward determined by the population density at her current position. The transition function is deterministic that directs an agent to the next state with probability one: $$P(s_{t+1} \vert s_t, a_t, \mu_t) = \mathds{1}_{\{s_{t+1} = a_t\}}.$$ This configuration is illustrated below. 
 	
	\begin{figure*}[!h]
		\centering
		\includegraphics[width=.6\textwidth]{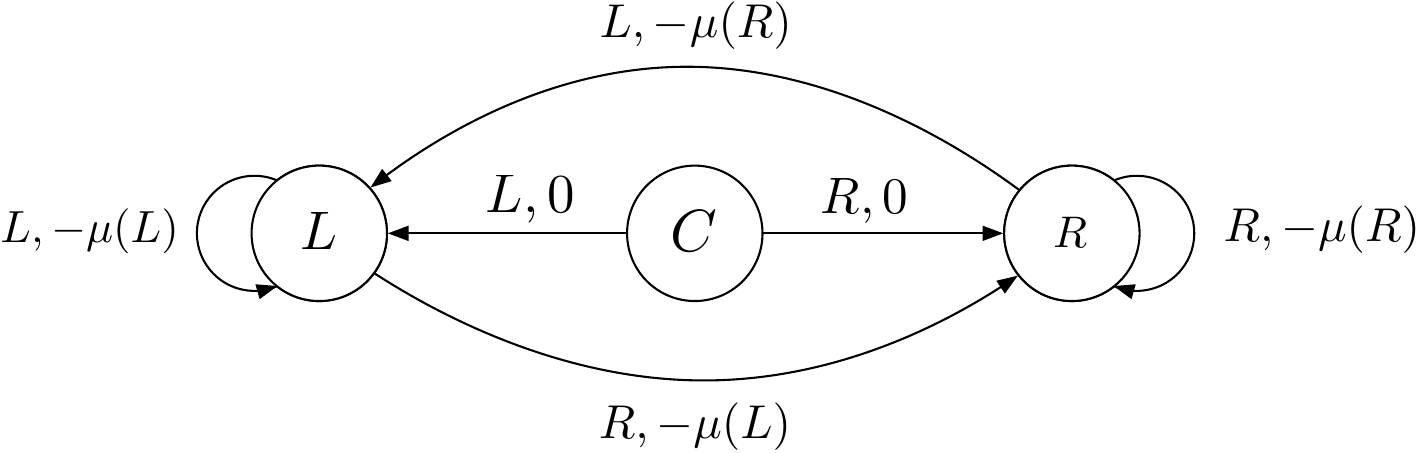}
	\end{figure*}
	
	Now, we consider the case that the time horizon $=1$.
	
	Since all agents are at ``center'' initially, we have that $\mu_1^\star(L) = \pi_0^\star(L \vert C)$ and $\mu_1^\star(R) = \pi_0^\star(R \vert C)$. Therefore, we have that the expected return of each agent under MFNE is 
	\begin{equation*}
	\begin{aligned}
		&~ -1 \cdot \pi_0^\star(L \vert C) \cdot \mu_1^\star(L) - 1 \cdot \pi_0^\star(R \vert C) \cdot \mu_1^\star(R)\\
		= &~ - (\pi_0^\star(L \vert C))^2 -  (1 - \pi_0^\star(L \vert C))^2 \\
		= &~ -\left( 2(\pi_0^\star(L \vert C))^2 - 2 \pi_0^\star(L \vert C) + 1 \right).
	\end{aligned}
	\end{equation*}
	
	Clearly, the expected return attains the maximum when $\pi_0^\star(L \vert C) = 1/2$. Therefore, any MFNE $(\muv^\star, \piv^\star)$ under the configuration above must fulfil $\pi^\star_0(R \vert C) = \pi^\star_0 (L \vert C) = 1/2$ and $\pi^\star_1$ can be arbitrary. And clearly, there exists a unique ERMFNE $(\tilde{\muv}^\star, \tilde{\piv}^\star)$ such that any action at any state and any time step is chosen with probability $1/2$. This result is also shown as a case study in \citep{cui2021approximately}. 
	
	Next, we change the reward function by adding an additional reward based on {\em actions} to the original reward function. We penalise the action ``left'' by a negative value $-1$, i.e., $$r'(s,a,\mu_t) = r(s,a,\mu_t) - \mathds{1}_{\{a = L\}} = -\mathds{1}_{\{s = L\}}\cdot \mu_t(L) -\mathds{1}_{\{s = R\}}\cdot \mu_t(R) - \mathds{1}_{\{a = L\}}.$$
	This equivalent to that an action-based reward shaping function $F(s_t,a_t,\mu_t,s_{t+1},\mu_{t+1}) = g(s_{t+1}, a_{t+1}, \mu_{t+1}) - g(s_t, a_t, \mu_t)$ is added to the original reward function where $$g(s_t, a_t, \mu_t) = - \mathds{1}_{\{a_t = L\}}.$$ The following diagram shows this new reward configuration.
	\begin{figure*}[!h]
		\centering
		\includegraphics[width=.6\textwidth]{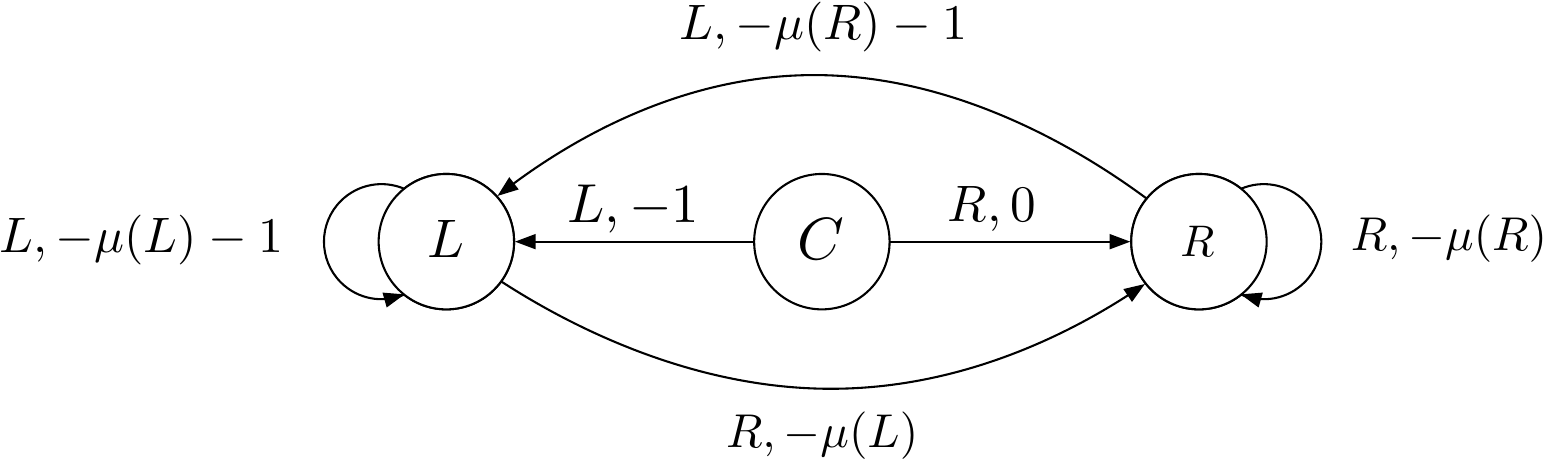}
	\end{figure*}
	
	Now, let us investigate the form of MFNE and ERMFNE under this new reward configuration. We first show the set of MFNE induced by the new reward function is no longer the same as that induced by the original reward function. Consider the step $t=1$ (the last step), since the reward of moving right is always higher than moving left by 1 and MFNE aims to maximise cumulative rewards, all agent will move right, i.e., $\pi_1^\star(R \vert L) = \pi_1^\star(R \vert R) = 1$. Hence, each agent earns a reward $-\mu_1^\star(L)$ if she is at ``left'' and $-\mu_1^\star(R)$ otherwise. Using the fact that $\mu_1^\star(L) = \pi_0^\star(L \vert C)$ and $\mu_1^\star(R) = \pi_0^\star(R \vert C)$, we have that the expected return of each agent under MFNE is 
	\begin{equation*}
	\begin{aligned}
		&~ -1 \cdot \pi_0^\star(L \vert C) + 0 \cdot  \pi_0^\star(R \vert C) -\mu_1^\star(L) \cdot \pi_0^\star(L \vert C) - \mu_1^\star(R) \cdot  \pi_0^\star(R \vert C)\\
		= &~ - \pi_0^\star(L \vert C)  -  (\pi_0^\star(L \vert C))^2 - (\pi_0^\star(R \vert C))^2\\
		= &~ - \pi_0^\star(L \vert C)  -  (\pi_0^\star(L \vert C))^2 - (1 - \pi_0^\star(L \vert C))^2\\
		= &~ - \left( 2 (\pi_0^\star(L \vert C))^2 - \pi_0^\star(L \vert C) + 1 \right).
	\end{aligned}
	\end{equation*}
	
	From here, we have that the expected return attains the maximum when $\pi_0^\star(L \vert C) = 1/4$, contradicting the MFNE induced by the original reward function.

	Next, we show that the ERMFNE induced by the new reward function also changes. Again, consider the last step. According to the definition of ERMFNE, we have 
	\begin{equation*}
		\tilde{\pi}^\star_1(L \vert L) = \frac{e^{-\tilde{\mu}_1^\star(L)}}{e^{-\tilde{\mu}_1^\star(L)} + e^{-\tilde{\mu}_1^\star(L) - 1}}. 
	\end{equation*}
	
	Therefore, $\tilde{\pi}^\star_1(L \vert L) = 1 / 2$ if and only if $\tilde{\mu}_1^\star(L) = \tilde{\mu}_1^\star(L) + 1$. A contradiction occurs. 
\end{proof}

\section{Descriptions for Numerical Tasks}\label{app:task}

The detailed descriptions below are adapted from \citep{chen2022individual}.

\subsection{Investment in Product Quality}

{\bf Model.} This model is adapted from \citep{weintraub2010computational} and \citep{subramanian2019reinforcement} that captures the investment decisions in a fragmented market with a large number of firms. Each firm produces the same kind of product. The state of a firm $s \in \Smc = \{ 0,1, \ldots, 9\}$ denotes the product quality. At each step, each firm decides whether or not to invest in improving the quality of the product. Thus the action space is $\Amc = \{0, 1\}$. When a firm decides to invest, its product quality increases uniformly at random from its current value to the maximum value 9 if the average quality in the market for that product is below a particular threshold $q$. If this average quality value is above $q$, then the product quality gets only half of the improvement compared to the former case. This implies that when the average quality in the economy is below $q$, it is easier for each firm to improve its quality. When a firm does not invest, its product quality remains unchanged. Formally, the dynamics is given by:

\begin{equation*}
	s_{t+1} = \left\{
	\begin{aligned}
		 & s_{t} + \lfloor \chi_t  ( 10- s_{t} ) \rfloor, \text{~~if~~}\langle \mu_{t} \rangle < q \text{~~and~~} a_{t} = 1\\
		 & s_{t} + \lfloor \chi_t  ( 10- s_{t} ) /2 \rfloor, \text{~~if~~}\langle \mu_{t} \rangle \geq q\text{~~and~~}a_{t} = 1\\
		 & s_{t}, \text{~~if~~}a_{t} = 0
	\end{aligned}\right..
\end{equation*} 

An agent incurs a cost due to its investment and earns a positive reward due to its own product quality and a negative reward due to the average product quality, which we denote by 
\begin{equation}\label{eq:average_mf}
\langle \mu_t \rangle \triangleq \sum_{s \in \Smc} s \cdot \mu_t(s).    
\end{equation}
The final reward is given as:
\begin{equation*}
	r(s_t,a_t,\mu_t) =  d \cdot s_{t} / 10 - c \cdot \langle \mu_{t} \rangle - \alpha \cdot a_{t} 
\end{equation*}

\subsubsection{Settings.}  We set $d=0.3$, $c=0.2$, $\alpha=0.2$ and probability density $f$ for $\chi_t$ as $U(0,1)$. We set the threshold $q$ to $4$ and $5$, respectively. The initial mean field $\mu_0$ is set as a uniform distribution, i.e, $\mu_0(s) = 1 / |\Smc|$ for all $s \in \Smc$.

\subsection{Malware Spread}

\subsubsection{Model.}  The malware spread model is presented in \citep{huang2016mean,huang2017mean} and used as a simulated study for MFG in \citep{subramanian2019reinforcement}. This model is representative of several problems with positive externalities, such as flu vaccination and economic models involving the entry and exit of firms. %collusion among firms, mergers, advertising, investment, network effects, durable goods, and consumer learning. 
Here, we present a discrete version of this problem:  Let $\Smc = \{0, 1, \ldots, 9\}$ denote the state space (level of infection), where $s = 0$ is the most healthy state and $s = 9$ is the least healthy state. The action space $\Amc = \{0, 1\}$, where $a = 0$ means $\mathtt{Do Nothing}$ and $a = 1$ means $\mathtt{Intervene}$. The dynamics is given by
\begin{equation*}
	s_{t+1} = \left\{
	\begin{aligned}
		& s_t + \lfloor \chi_t  ( 10- s_t ) \rfloor,  \text{~~if~~}  a_t = 0\\
		& 0, \text{~~if~~}a_t = 1
	\end{aligned}\right.,
\end{equation*}
where $\{\chi_t\}_{0\leq t \leq T}$ is a $[0, 1]$-valued i.i.d. process with probability density $f$. The above dynamics means the $\mathtt{Do Nothing}$ action makes the state deteriorate to a worse condition, while the $\mathtt{Intervene}$ action resets the state to the most healthy level. Rewards are coupled through the average health level of the population, i.e., $\langle \mu_{t} \rangle$ as defined in Eq.\eqref{eq:average_mf}. An agent incurs a cost $(k + \langle \mu_t \rangle) s_t$, which captures the risk of getting infected, and an additional cost of $\alpha$ for performing the $\mathtt{Intervene}$ action. The reward sums over all negative costs:
\begin{equation*}
	r(s_t, a_t, \mu_t) = -(k + \langle \mu_t \rangle)s_t/10 - \alpha \cdot a_t.
\end{equation*}

\subsubsection{Settings.} Following \citep{subramanian2019reinforcement}, we set $k = 0.2$, $\alpha = 0.5$, and the probability density $f$ to the uniform distribution $U(0,1)$ for the original dynamics. The initial mean field $\mu_0$ is set as a uniform distribution.

\subsection{Virus Infection} 

\subsubsection{Model.} This is a virus infection used as a case study in \citep{cui2021approximately}. There is a large number of agents in a building. Each can choose between ``social distancing'' ($D$) or ``going out'' ($U$). If a ``susceptible'' ($S$) agent chooses social distancing, they may not become ``infected'' ($I$). Otherwise, an agent may become infected with a probability proportional to the number of infected agents. If infected, an agent will recover with a fixed chance every time step. Both social distancing and being infected have an associated negative reward.
Formally, let $\Smc = \{S,I\}, \Amc = \{U,D\}, r(s,a, \mu_t) = -\mathds{1}_{\{s = I\}} - 0.5 \cdot \mathds{1}_{\{s = D\}}$. The transition probability is given by
\begin{equation*}
	\begin{aligned}
		P(s_{t+1} = S \vert s_t = I, \cdot, \cdot) & = 0.3\\
		P(s_{t+1} = I \vert s_t = S, a_t = U, \mu_t) & = 0.9^2 \cdot \mu_t(I)\\
		P(s_{t+1} = I \vert s_t = S, a_t = D, \cdot) &= 0.
	\end{aligned}
\end{equation*}

\subsubsection{Settings.} The initial mean field $\mu_0$ is set as a uniform distribution.

\subsection{Rock-Paper-Scissors} This model is adapted by \citep{cui2021approximately} from the generalized non-zero-sum version of {\em Rock-Paper-Scissors} game \citep{shapley1964some}. Each agent can choose between ``rock'' ($R$), ``paper'' ($P$) and ``scissors'' ($S$), and obtains a reward proportional to double the number of beaten agents minus the number of agents beating the agent. Formally, let $\Smc = \Amc = \{R, P, S\}$, and for any $a \in \Amc, \mu_t \in \Pmc(\Smc)$:
\begin{equation*}
	\begin{aligned}
		r(R, a, \mu_t) &= 2 \cdot \mu_t(S) - 1\cdot \mu_t(P),\\
		r(P, a, \mu_t) &= 4 \cdot \mu_t(R) - 2 \cdot \mu_t(S),\\
		r(S, a, \mu_t) &= 6 \cdot \mu_t(P) - 3 \cdot \mu_t(R).
	\end{aligned}
\end{equation*}

The transition function is deterministic: $p(s_{t+1} \vert s_t, a_t, \mu_t) = \mathds{1}_{\{s_{t+1} = a_t\}}.$

\subsubsection{Settings.} The initial mean field $\mu_0$ is set as a uniform distribution.

\subsection{Left-Right}

\subsubsection{Model.} This model is used in \citep{cui2021approximately}. A group of agents makes sequential decisions to move ``left'' or ``right''. At each step, each agent is at a position (state) either ``left'', ``right'' or ``center'', and can choose to move either ``left'' or ``right'', receives a reward according the current population density (mean field) at each position, and with probability one (dynamics) they reach ``left'' or ``right''. Once an agent leaves ``centre'', she can never head back and can only be on the left or right thereafter. Formally, we configure the MFG as follows: $\Smc = \{C, L, R\}$, $\Amc = \Smc \setminus \{C\}$, the reward $$r(s,a,\mu_t) = -\mathds{1}_{\{s = L\}}\cdot \mu_t(L) -\mathds{1}_{\{s = R\}}\cdot \mu_t(R).$$ This reward setting means each agent will incur a negative reward determined by the population density at her current position. The transition function is deterministic that directs an agent to the next state with probability one: $P(s_{t+1} \vert s_t, a_t, \mu_t) = \mathds{1}_{\{s_{t+1} = a_t\}}.$

\subsubsection{Settings.} The initial mean field $\mu_0$ is set as $\mu_0(L) = \mu_0(R) = 0.5$. 

\section{Implementation Details}\label{app:exp}

The experimental settings below are adapted from \citep{chen2022individual}.

\subsection{Feature representations.} We use one-hot encoding to represent states and actions. Let $\{1,2,\ldots, |\Smc|\}$ denote an enumeration of $\Smc$ and $\left[s_{[1]}, s_{[2]}, \ldots, s_{[|\Smc|]}\right]$  denote a vector of length $|\Smc|$, where each component stands for a state in $\Smc$. The state $j$ is denoted by $\big[0, \ldots, 0, s_{[j]} = 1, 0,$ $\ldots, 0 \big]$. An action is represented in the same manner. A mean field $\mu$ is represented by a vector $\left[\mu(s_{[1]}), \mu(s_{[2]}), \ldots, \mu(s_{[|\Smc|]})\right]$, where $\mu(s_{[i]})$ denotes the proportion of agents that are in the $i$th state. 

\subsection{Reward Models and Adaptive Samplers.} The reward mode $r_\omega$ takes as input the concatenation of feature vectors of $s$, $a$ and $\mu$ and outputs a scalar as the reward. We adopt the neural network (a four-layer perceptron) with the Adam optimiser and the Leaky ReLU activation function. The sizes of the two hidden layers are both 64. The learning rate is $10^{-4}$. The adaptive sampler $\pi^{\theta_t}$ ($0 \leq t < T$) takes as input the feature vector of a state and outputs a distribution over the action set. The neural network architecture for each adaptive sampler is a five-layer perceptron with the Adam optimiser and the Leaky ReLU activation function. The sizes of the first two hidden layers are both 64, the size of the third hidden layer is identical to the size of the action set, and the last layer is a softmax layer for generating a distribution over the action set.

\subsection{Computation of ERMFNE.} %Note that the learning algorithms  
In ERMFNE expert training, we repeat the fixed point iteration to compute the MF flow. We terminate at the $i$th iteration if the mean squared error over all steps and all state is below or equal to $10^{-10}$, i.e., $$\frac{1}{(T-1)|\Smc|} \sum_{t=1}^{T-1}\sum_{s \in \Smc} \left(\mu^{(i)}_t(s) - \mu^{(i-1)}_t(s)\right)^2 \leq 10^{-10}.$$
\end{document}